\documentclass[draftclsnofoot,onecolumn, 12pt]{IEEEtran}
\usepackage{color}
\usepackage{graphicx, pstool}
\usepackage{amsmath}
\usepackage{amsthm}
\usepackage{amsfonts}
\usepackage{amssymb}
\usepackage{booktabs}
\usepackage{dsfont}
\usepackage{cases}
\usepackage{tabularx}
\usepackage{boxedminipage}
\usepackage{mathrsfs}
\usepackage{tikz}
\usepackage{caption}
\usepackage[margin=20mm,font={footnotesize},labelfont={footnotesize}]{caption}
\usepackage{cite}

\newcommand{\pr}{{\rm {Pr}}}

\newcommand{\e}{{\rm e}}

\def\@footnotetext{\insert\footins\bgroup\@makeother\#\do@footnotetext}
\newcommand{\ttvar}{\begingroup\@makeother\#\@ttvar}
\makeatother

\newcommand\myeqa{\mathrel{\stackrel{\makebox[0pt]{\mbox{\normalfont\tiny $(a)$}}}{=}}}
\newcommand\myeqb{\mathrel{\stackrel{\makebox[0pt]{\mbox{\normalfont\tiny $(b)$}}}{=}}}
\newcommand\myleqa{\mathrel{\stackrel{\makebox[0pt]{\mbox{\normalfont\tiny $(a)$}}}{\leq}}}
\newcommand\myleqb{\mathrel{\stackrel{\makebox[0pt]{\mbox{\normalfont\tiny $(b)$}}}{\leq}}}
\newcommand\myleqc{\mathrel{\stackrel{\makebox[0pt]{\mbox{\normalfont\tiny $(c)$}}}{\leq}}}
\newcommand\mygeqa{\mathrel{\stackrel{\makebox[0pt]{\mbox{\normalfont\tiny $(a)$}}}{\geq}}}
\newcommand\mygeqb{\mathrel{\stackrel{\makebox[0pt]{\mbox{\normalfont\tiny $(b)$}}}{\geq}}}
\newcommand\myleqqb{\mathrel{\stackrel{\makebox[0pt]{\mbox{\normalfont\tiny $(b)$}}}{<}}}

\newcommand\bigexists{%
	\mathop{\lower0.75ex\hbox{\ensuremath{%
				\mathlarger{\mathlarger{\exists}}}}}%
	\limits}

\newtheorem{theorem}{Theorem}

\newtheorem{cor}{Corollary}
\newtheorem{por}{Proposition}

\begin{document}

\title{On Supervised Classification of Feature Vectors with Independent and Non-Identically Distributed Elements}
\author{Farzad Shahrivari and Nikola Zlatanov
\thanks{Farzad Shahrivari and Nikola Zlatanov are  with the Department of Electrical and Computer Systems Engineering, Monash University, Melbourne, VIC $3800$, Australia.  (email: farzad.shahrivari@monash.edu, nikola.zlatanov@monash.edu).}
\vspace{-5mm}
}

\maketitle
\begin{abstract}
In this paper, we investigate the problem of classifying  feature vectors with mutually independent but non-identically distributed elements. First, we show the importance of this problem. Next,  we propose a classifier and   derive an analytical upper bound on its  error probability. We show that the error probability goes to zero as  the length of the feature vectors  grows, even when there is only one training feature vector per label available. Thereby, we show that for this important problem at least one  asymptotically optimal classifier exists. Finally, we provide numerical examples where we show that the performance of the proposed classifier outperforms conventional classification algorithms when the number of training data is small and the length of the feature vectors is sufficiently high.  
\end{abstract}


\IEEEpeerreviewmaketitle
\section{Introduction}

\subsection{Background}

\IEEEPARstart{S}upervised classification is a machine learning technique that maps an input feature vector  to an output label based on  a set of correctly labeled 
training data. 
There is no single learning algorithm that works best on all supervised learning problems,  as shown by the no free lunch theorem in \cite{10.5555/2621980}. As a result, there are many algorithms proposed in the literature whose performance depends on the underlying problem and the amount of  training data available. The most widely used algorithms in the literature are decision trees \cite{Quinlan1983,reason:BreFriOlsSto84a}, Support Vector Machines (SVM) \cite{10.1145/130385.130401,10.1023/A:1022627411411}, Rule-Based Systems \cite{inbook}, naive Bayes classifiers \cite{10.5555/1867135.1867170}, k-nearest neighbors (KNN) \cite{10.5555/273530.273534}, logistic regressions, and neural networks \cite{10.5555/525960,hastie2009elements}.

\subsection{Motivation}

In the following, we discuss the motivation for this work.

\subsubsection{Lack of Tight Upper Bounds on the Performance of  Classifiers} 

In general, there are no tight upper bounds on the performance of the classifiers used in practice.  Many of the previous works only provide experimental performance results. However, this approach has  drawbacks. For example,   one has to rely on the  trail-and-error approach  in order to develop a good classifier for a given problem. Next, the algorithms whose performance has been verified only experimentally may work for a given problem, but may fail to work  when applied to a similar problem. Finally, experimental results do not provide intuition into the underlying problem, whereas  the analytical results   provide the   understanding of the underlying problem and the corresponding solutions.

Motivated by this, in the paper, we aim to investigating classifiers with analytical upper bounds on their performance.

\subsubsection{Independent and Non-Identically Distributed Features} 

In general, we can categorize the statistical properties of the  feature vectors, which are the input to the classifier, into three types. To this end, let $Y^n(X)=\big[Y_1(X),Y_2(X),\ldots,Y_n(X)\big]$ denote the input feature vector to the supervised classifier, where $n$ is the length of the feature vector and  $X$ is the label to which the feature vector $Y^n(X)$ belongs. Then, we can distinguish the following three types of feature vectors depending on the statistics of the elements in the feature vector  $Y^n(X)$. 

The first type of feature vectors is   when the elements of $Y^n(X)$ are independent and identically distributed (i.i.d). This is the simplest features model, but also   the least applicable in practice. This model is identical to hypothesis testing, which has been well investigated in the literature \cite{476321,32134,8052528}. As a result,  tight upper bounds on the  performance of supervised learning algorithms for this type of feature vectors are available in the hypothesis testing literature.  For instance, the authors in \cite{476321} showed that the posterior entropy and the maximum a posterior error probability decay to zero with the length of the feature vector at the identical exponential rate, where the maximum achievable exponent is the minimum Chernoff information. In \cite{32134}, the authors determine the requirements for the length of the vector $Y^n(X)$ and the number of labels $m$ in order to achieve vanishing exponential error probability in testing $m$ hypothesis that minimizes the rejection zone. In \cite{8052528}, the authors provide an upper bound and a lower-bound on the error probability of Bayesian $m$-ary hypothesis testing in terms of conditional entropy.

The second type of  feature vectors is when the elements of  $Y^n(X)$ are mutually dependent and non-identically distributed (d.non-i.d.). This type of features model is the most  general   model and the most applicable   in practice. However,  it is also the most difficult to tackle analytically. As a result,   supervised learning algorithms proposed for this features model    lack   analytical tight upper bounds on their performance \cite{10.5555/2540128.2540377,Wang2015CoupledAS,10.1145/2063576.2063715,Liu2015ACK,inproceedings2,6889773, 10.1093/comjnl/bxt084,CAO2015167,Cao2011CombinedMD,10.5555/1296231}.  This is  because there aren't any frameworks that produce closed-form results when deriving statistics of vectors with d.non-i.d. elements  when the underlying distributions are unknown.  Then how can we investigate analytically classifiers  for the practical scenarios  when the feature vectors have d.non-i.d. elements? A possible approach leads us to the third type of feature vectors, explained in the following.

The third  type of feature vectors is when the   elements of $Y^n(X)$ are mutually independent but non-identically distributed (i.non-i.d.). This features model is much simpler than the d.non-i.d. features model  and,   more importantly, it is analytically tractable, as we show in this paper. Furthermore, this features model is applicable in practice. Specifically, there exists a class  of algorithms, known as Independent Component Analysis (ICA),  that transform  vectors with d.non-i.d. elements into   vectors with i.non-i.d. elements with a zero or a negligible loss of information   \cite{10.1109/18.556108,article1111,10.1016/S0893-6080(00)00026-5,10.1109/TIT.2011.2145090,Nguyen_2011,7362043}.   The origins of ICA can be traced back to   Barlow \cite{10.1162/neco.1989.1.3.295}, who
argued that a good representation of binary data   can be achieved by an invertible transformation  that  transform  vectors with d.non-i.d. elements into   vectors with i.non-i.d. elements. Finding such a transformation with no prior information about the distribution of the data has been considered an open problem until recently \cite{7362043}. Specifically,  the authors in\cite{7362043}   show that this hard problem can be accurately solved with a branch and bound search tree algorithm, or tightly approximated with a series of linear problems. Thereby, the authors in \cite{7362043}   provide  the first efficient set of solutions to Barlow's problem. So far, the complexity of the fastest such  algorithm is $\mathcal{O}\big(n\times2^n\big)$ \cite{7362043}. Nevertheless, since there exist such invertible transformations (i.e., no loss of information) which can transform   vectors with d.non-i.d. elements into   vectors with i.non-i.d. elements,  we can tackle the  features model comprised of d.non-i.d. elements  by  first  transforming it (without loss of information) into the  features model comprised of i.non-i.d. elements and then tackling the  i.non-i.d. features model.

Motivated by   this, in this paper, we investigate supervised classification of feature vectors with i.non-i.d. elements.

\subsubsection{Small Training Set} 
The main factor that impacts the accuracy of   supervised classification is   the amount of  training data. In fact, most supervised algorithms are able to learn only if there is a very large set of training data available \cite{10.1007/978-3-642-59051-1_13}.
The main reason for this is the 
 curse of dimensionality \cite{thecurse,10.5555/954544}, which states that ``the higher the dimensionality of the feature vectors, the more training data is needed for the supervised classifier''\cite{10.1109/34.824819}. For example,  supervised classification methods such as random forest \cite{article13,YE2013769} and KNN\cite{6065061} suffer from the curse of dimensionality. However, having large training data sets is not always possible in practice. As a result,    designing a  supervised classification algorithm that exhibits good performance even when the training data set is extremely small is   important. 
 
 Motivated by this, in this paper, we investigate supervised classifiers for the case  when $t$ training feature vectors per label are available, where $t=1,2,...$

\subsection{Contributions}
In this paper, we propose an algorithm for  supervised classification of feature vectors with i.non-i.d. elements when the number of training feature vectors per label is $t$, where $t=1,2,...$  Next, we derive an upper bound on the error probability of the proposed classifier for uniformly distributed labels and prove that the error probability exponentially decays to zero when the length of the feature vector, $n$, grows, even when only one training vector per label is available, i.e., when $t=1$. Hence, the proposed classification algorithm provides an asymptotically optimal performance even   when the number of training vectors per label is extremely small.  We compare the performance of the proposed classifier with   the naive Bayes classifier 
 and to the KNN algorithm. Our numerical results show that the proposed classifier significantly outperforms the naive Bayes classifier and the KNN algorithm  when the number of training feature vectors per label is small and the length of the feature vectors $n$ is sufficiently high. 

The proposed algorithm is a form of the nearest neighbour classification algorithm, where the nearest neighbour is searched in the domain of empirical distributions. As a result, we refer to the algorithm as   \textit{the nearest   empirical distribution}. The nearest  empirical distribution algorithm is not new  and, to the best of our knowledge, it was first proposed in \cite{matusita1967} for the   case when the elements of $Y^n(X)$ are i.i.d., i.e., for the equivalent problem of hypothesis testing. However, in this paper, we propose the nearest  empirical distribution algorithm for the    case when the elements of $Y^n(X)$ are i.non-i.d., which is much more complex than the problem of hypothesis testing where the elements of $Y^n(X)$ are i.i.d.

 
 The main contributions of this paper are as follows:

\begin{itemize}
 
 \item We show that the   problem  of classifying feature vectors with i.non-i.d. elements is an important problem in machine learning, both from a theoretical and from  a practical point of view. Surprisingly, although important, this problem has not be tackled in the literature yet. Therefore, we are  the first to introduce and tackle this important problem.

\item For this  important problem  of classifying feature vectors with i.non-i.d. elements, we propose a classifier and derive an  analytical upper bound on its error probability. Thereby, we show that this problem can be tackled analytically.

\item Next, we show that that the proposed classifier is asymptotically optimal since its  error probability goes to zero as the length of the feature vector grows. As a result, we show that for this problem at least one asymptotically optimal classifier   exists. Intuitively, there must exist  other classifiers that would exhibit  better non-asymptomatic performance than the proposed classifier. Due to the importance of the proposed problem,  other researchers might investigate these classifiers and this paper aims to serve as the motivation for this future research.

\item   To the best of our knowledge, this is the first   classifier that exhibits an error probability that goes to zero as the length of the feature vector  grows to infinity even if there is  one training feature vector per label available.   Previous classifiers were designed to exhibit optimal performance when the number of training feature vectors per label grows to infinity, and usually these classifiers encounter  problems when the number of features per label goes to infinity due to the curse of dimensionality. However, here, we show that having more independent features is not a curse and in fact is a blessing. This is because   independent features    provide   independent descriptions of the label. As a result, intuitively, the more independent features/descriptions a label can have, the more accurate the classification can be.

\end{itemize}

The remainder of this paper is structured as follows. In Sec. \ref{sec2}, we formulate the considered classification problem. In Sec. \ref{sec3}, we provide our classifier and derive an upper bound on its error probability. In Sec. \ref{sec4}, we provide numerical examples of the performance on the proposed classifier. Finally, Sec. \ref{sec5} concludes the paper.

\section{Problem Formulation}\label{sec2}

The classification learning model is comprised of a label\footnote{In this paper, we adopt the information-theoretic style of notations and thereby random variables are denoted by capital letters and their realizations are denoted with small letters.}
 $X$, a feature vector $Y^n(X)=\big[Y_1(X),Y_2(X),$ $\ldots,Y_n(X)\big]$ of length $n$ mapped to the label $X$, and a learned\textbackslash detected label $\hat{X}$, as shown in Fig.~\ref{fig:1.1}. The feature vector $Y^n(X)$ is the input to the classification learning algorithm whose aim is to detect the label $X$ from the observed feature vector $Y^n(X)$. The performance of the classification learning algorithm is measured by the error probability $\mathbb{P}_{\e}=\pr\big\{X\neq\hat{X}\big\}$.
 
\begin{figure} 
	\centering			\includegraphics[width=0.5\linewidth]{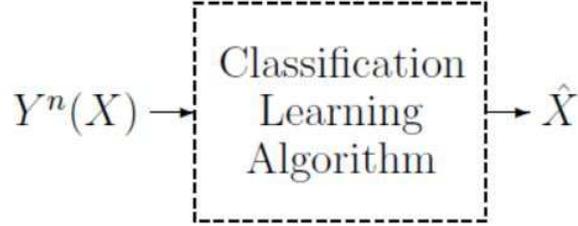}
	\caption{A typical structural modelling of the classification learning problem} \label{fig:1.1}
\end{figure}

We adopt the modelling in \cite{10.1023/A:1013912006537,Deisenroth2020,Lebanon02cranking:combining} and represent the dependency between the label $X$ and the feature vector $Y^n(X)$ via a  joint probability distribution $p_{X,Y^n}(x,y^n)$. Now, in order to have a better understanding of the classification learning problem, we   include the  joint probability distribution $p_{X,Y^n}(x,y^n)$ into the model in Fig.~\ref{fig:1.1}. To this end, 
since $p_{X,Y^n}(x,y^n)=p_{Y^n|X}(y^n|x) p_X(x)$ holds, instead of including $p_{X,Y^n}(x,y^n)$ into Fig.~\ref{fig:1.1}, we can include the    conditional probability distribution $p_{Y^n|X}(y^n|x)$ and the probability distribution $p_X(x)$ into the model in Fig.~\ref{fig:1.1}, and thereby obtain the model in Fig.~\ref{fig:1.2}. 
\begin{figure} 
	\centering			\includegraphics[width=1\linewidth]{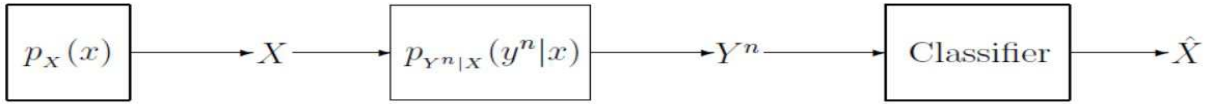}
	\caption{An alternative modeling of the classification learning problem.} \label{fig:1.2}
\end{figure}

Now, the classification learning model in Fig.~\ref{fig:1.2} is a system\footnote{Note that the system model in Fig.~\ref{fig:1.2} can be seen equivalently as a communication system comprised of a source $X$, a channel with input $X$ and output $Y^n$, and a decoder (i.e., detector) that aims to detect $X$ from $Y^n$. The notation used in this paper, letter $X$ for labels and letter $Y$ for features, is based on the notation used in information theory for modelling communication systems.} comprised of a label generating source $X$ according to the distribution $p_{_X}(x)$, a feature vector generator modelled by the conditional probability distribution $p_{_{Y^n|X}}(y^n|x)$, a feature vector $Y^n$, a classifier that aims to detect $X$ from the observed feature vector $Y^n$, and the detected label $\hat{X}$.

 In the classification model shown in Fig.~\ref{fig:1.2}, we assume that the label $X$ can take values from the set $\mathcal{X}$, according to $p_{_X}(x)=\dfrac{1}{|\mathcal{X}|}$, where $|\cdot|$ denotes the cardinality of a set. Next, we assume that the $i$-th element of the feature vector $Y^n$, $Y_i$, for $i=1,2,\dots,n$, takes values from the set $\mathcal{Y}=\big\{y_1,y_2,\ldots,y_{|\mathcal{Y}|}\big\}$, according to the conditional probability distribution $p_{_{Y_i|X}}(y_i|x)$. Moreover, we assume that the elements of the feature vector $Y^n$ are i.non-i.d. As a result, the feature vector $Y^n$ takes values from the set $\mathcal{Y}^n$ according to the conditional probability distribution $p_{_{Y^n|X}}(y^n|x)$ given by
\begin{align}\label{eq_e1}
	p_{_{Y^n|X}}(y^n|x)&=p_{_{Y_1,Y_2,\ldots,Y_n|X}}(y_1,y_2,\ldots,y_n|x)\nonumber\\&\myeqa\prod_{i=1}^np_{_{Y_i|X}}(y_i|x)\nonumber\\
	&\myeqb\prod_{i=1}^np_{_{i}}(y_i|x),
\end{align}
where $(a)$ comes from the fact that elements in the feature vector $Y^n$ are mutually independent and $(b)$ is for the sake of notational simplicity, where $p_{_i}$ is used instead of $p_{_{Y_i|X}}$. As a result of \eqref{eq_e1}, the considered classification model in Fig.~\ref{fig:1.2}  can be represented equivalently as in Fig.~\ref{fig:1.3}.

\begin{figure} 
	\centering			\includegraphics[width=1\linewidth]{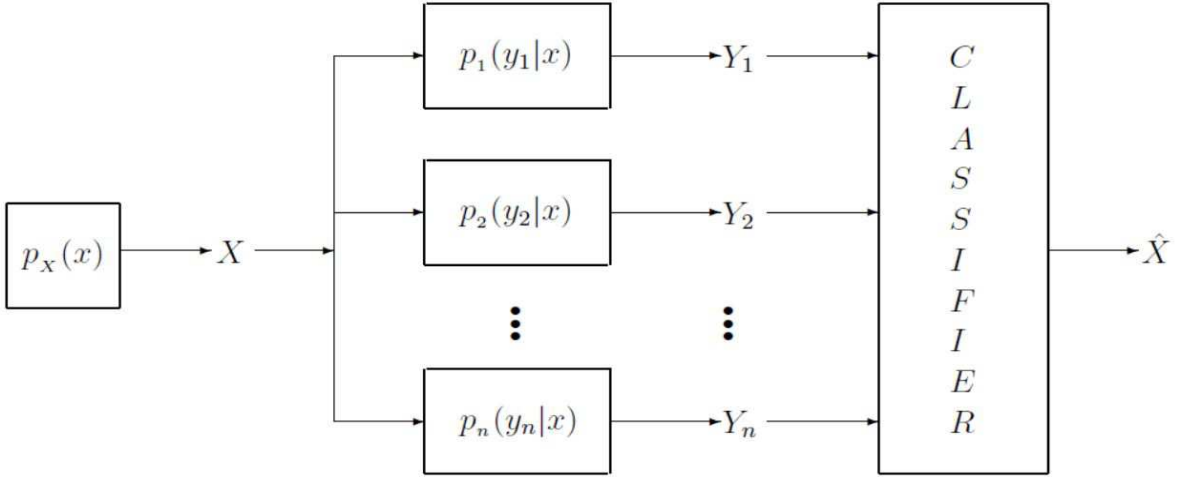}
	\caption{An alternative modelling of the classification learning problem when the elements of $Y^n(X)$ are mutually independent but non-identically distributed (i.non-i.d.).} \label{fig:1.3}
\end{figure}
Next, we assume that $p_{_i}(y_i|x)$, $\forall i$, and thereby $p_{_{Y^n|X}}(y^n|x)$, are unknown to the classifier. Instead, the classifier knows $\mathcal{X}$, $\mathcal{Y}$, and for each $x_i\in\mathcal{X}$, where $i=1,2,\ldots,|\mathcal{X}|$, it has access to a finite set of $t$ correctly labelled input-output pairs $(x_i,\hat{y}_{i_1}^n), (x_i,\hat{y}_{i_2}^n),\ldots, (x_i,\hat{y}_{i_t}^n)$, denoted by $\mathcal{T}_i$, referred to as the training set for label $x_i$. 

For the  classification system model defined above and illustrated in Fig.~\ref{fig:1.3}, we wish to propose a classifier that exhibits an asymptotically optimal error probability $\mathbb{P}_{\e}=\pr\big\{X\neq\hat{X}\big\}$     with respect to the length of $Y^n$, $n$, for any $t\geq 1$, i.e., for any $t\geq 1$, $\mathbb{P}_{\e}\to 0$ as $n\to\infty$. Moreover, we wish to obtain an analytical upper bound on the error probability of the proposed classifier for a given $t$ and $n$.

\section{The Proposed Classifier and its Performance}\label{sec3}
In this section, we propose our classifier, derive an analytical upper bound on its error probability, and prove that the classifier exhibits an asymptotically optimal performance when the length of the feature vector $Y^n$, $n$, satisfies $n\to\infty$. This is done in the following.

For  given vectors $\mathrm{v}^n=(v_1,v_2,\ldots,v_n)$ and $\mathrm{u}^n=(u_1,u_2,\ldots,u_n)$  let the Minkowski distance $r$ be defined as 
\begin{align}
	\big\lVert \mathrm{v}-\mathrm{u}\big\rVert_r = \bigg(\sum_{i=1}^n (v_i-u_i)^r\bigg)^{(1/r)}.
\end{align}
Also, for a given feature vector $y^k=(y_1, y_2,,\ldots,y_k)$, let $\mathcal{I}[y^k=y]$ be a function defined as
\begin{align}\label{eq_e2}
	\mathcal{I}[y^k=y]=\sum_{i=1}^{k}\mathcal{Z}[y_i=y],	
\end{align}
where $\mathcal{Z}[y_i=y]$ is an indicator function assuming the value $1$ if $y_i=y$ and $0$ otherwise. Hence, $\mathcal{I}[y^k=y]$ counts the number of elements in $Y^k$ that have the value $y$.

\subsection{The Proposed Classifier}

Let $\hat{y}^{nt}_i$ be a vector obtained by concatenating all training feature vectors for the input label $x_i$ as 
\begin{align}\label{eq_e3}
\hat{y}^{nt}_i=\Big(\hat{y}^n_{i_1}, \hat{y}^n_{i_2}, \ldots,\hat{y}^n_{i_t}\Big).
\end{align}
Let $P_{\hat{y}^{nt}_i}$ be the empirical probability distribution of the concatenated training feature vector for label $x_i$, $\hat{y}^{nt}_i$, given by
\begin{align}\label{eq_e4}
P_{\hat{y}^{nt}_i}=\Bigg[\dfrac{\mathcal{I}\big[\hat{y}_i^{nt}=y_1\big]}{nt},\dfrac{\mathcal{I}\big[\hat{y}^{nt}_i=y_2\big]}{nt},\ldots,\dfrac{\mathcal{I}\big[\hat{y}^{nt}_i=y_{|\mathcal{Y}|}\big]}{nt}\Bigg].
\end{align}
Let $y^n$ be the observed feature vector at the classifier whose label the classifier wants to detect and let $P_{y^n}$ denote the empirical probability distribution of $y^n$, given by
\begin{align}\label{eq_e5}
P_{y^{n}}=\Bigg[\dfrac{\mathcal{I}\big[y^n=y_1\big]}{n},\dfrac{\mathcal{I}\big[y^n=y_2\big]}{n},\ldots,\dfrac{\mathcal{I}\big[y^n=y_{|\mathcal{Y}|}\big]}{n}\Bigg].
\end{align}
Using the above notations, we propose the following classifier.
\begin{por}
	For the considered system model, we propose a classifier with the following classification rule
	\begin{align}\label{eq_e6}
		\hat{x} =\arg\min_{x_i}\big\lVert P_{y^n}-P_{\hat{y}_i^{nt}}\big\rVert_r,
	\end{align}
	where $r\geq1$ and ties are resolved by assigning  the label among the ties uniformly at random.
\end{por}

As seen from \eqref{eq_e6}, the proposed classifier assigns the label $x_i$ if the empirical probability distribution of the concatenated training feature vector mapped to label $x_i$, $P_{\hat{y}_i^{nt}}$, is the closest, in terms of Minkowski distance $r$, to the empirical probability distribution of the observed feature vector $P_{y^n}$. In that sense, the proposed classifier can be considered as the nearest empirical distribution classifier.

\subsection{Upper Bound On The Error Probability}

The following theorem establishes an upper bound on the error probability of the proposed classifier. 
\begin{theorem}\label{th2}
	Let $\bar{\mathrm{P}}_j$, for $j=1,2,\ldots,|\mathcal{X}|$, be a vector defined as
	\begin{align}\label{eq_e7}
	\bar{\mathrm{P}}_j = \Big[\bar{\mathrm{p}}\big(y_1\big|x_j\big),\bar{\mathrm{p}}\big(y_2\big|x_j\big),\ldots,\bar{\mathrm{p}}\big(y_{|\mathcal{Y}|}\big|x_j\big)\Big],
	\end{align}
	where $\bar{\mathrm{p}}(y|x_j)$ is given by
	\begin{align}\label{eq_e8}
	\bar{\mathrm{p}}(y|x_j)=\dfrac{1}{n}\sum_{k=1}^n p_{_k}(y|x_j).
	\end{align}
	Then, for a given $r\geq1$, the error probability of the proposed classifier is upper bounded by
	\begin{align}\label{eq_e9}
	\mathbb{P}_\e\leq 2|\mathcal{Y}|\e^{-2n\epsilon^2}+2|\mathcal{Y}|\e^{-2nt^{1/3}\epsilon^2},
	\end{align}
	where $\epsilon$ is given by
	\begin{align}\label{eq_e10}
	\epsilon = \min_{\substack{i,j\\i\neq j}}\dfrac{\big\lVert P_{\hat{y}^{nt}_i}-\bar{\mathrm{P}}_j\big\rVert_r}{(2+t^{-1/3})|\mathcal{Y}|^{1/r}}.
	\end{align}
\end{theorem}

\begin{proof}
    Without loss of generality we assume that $x_1$ is the input to $p_{_{Y^n|X}}(y^n|x)$ and $y^n$ is observed. 
	
	Let $\mathcal{A}^{\epsilon}_k$, for $1\leq k\leq|\mathcal{Y}|$, be a set defined as
	\begin{align}\label{eq_e16}
	\mathcal{A}^{\epsilon}_k=\Bigg\{y^n:\bigg|\dfrac{\mathcal{I}\big[y^n=y_k\big]}{n}- \bar{\mathrm{p}}(y_k|x_1)\bigg|\leq\epsilon\Bigg\}.
	\end{align}
	
	Also, let $\mathcal{B}^{\epsilon}_k$, for $1\leq k\leq|\mathcal{Y}|$, be a set defined as
	\begin{align}\label{eq_e17}
	\mathcal{B}^{\epsilon}_k=\Bigg\{\hat{y}^{nt}:\bigg|\dfrac{\mathcal{I}\big[\hat{y}^{nT}=y_k\big]}{nt}- \bar{\mathrm{p}}(y_k|x_1)\bigg|\leq\frac{\epsilon}{\sqrt[3]{t}}\Bigg\}.
	\end{align}
	
	Let $\mathcal{A}^{\epsilon}=\bigcap\limits_{k=1}^{|\mathcal{Y}|} \mathcal{A}^{\epsilon}_k$ and $\mathcal{B}^{\epsilon}=\bigcap\limits_{k=1}^{|\mathcal{Y}|} \mathcal{B}^{\epsilon}_k$. Now, for any $y^n\in\mathcal{A}^{\epsilon}$, we have
	\begin{align}\label{eq_e18}
	\Bigg(\sum_{k=1}^{|\mathcal{Y}|}\bigg|\dfrac{\mathcal{I}[y^n=y_k]}{n}-\bar{\mathrm{p}}(y_k|x_1)\bigg|^{r}\Bigg)^{1/r}\myleqa \Bigg(\sum_{k=1}^{|\mathcal{Y}|}\epsilon^r\Bigg)^{1/r},
	\end{align}
	where $(a)$ follows from \eqref{eq_e16}. Moreover, for $\hat{y}_1^{nt}\in\mathcal{B}^{\epsilon}$, we have
	\begin{align}\label{eq_e19}
	\Bigg(\sum_{k=1}^{|\mathcal{Y}|}\bigg|\dfrac{\mathcal{I}[\hat{y}^{nt}_1=y_k]}{nt}-\bar{\mathrm{p}}(y_k|x_1)\bigg|^{r}\Bigg)^{1/r}\myleqa \Bigg(\sum_{k=1}^{|\mathcal{Y}|}\bigg(\frac{\epsilon}{\sqrt[3]{t}}\bigg)^r\Bigg)^{1/r},
	\end{align} 
	where $(a)$ follows from \eqref{eq_e17}. Next, we have the following upper bound
	\begin{align}\label{eq_e20}
	&\Bigg(\sum_{k=1}^{|\mathcal{Y}|}\bigg|\dfrac{\mathcal{I}[y^n=y_k]}{n}-\dfrac{\mathcal{I}[\hat{y}_1^{nt}=y_k]}{nt}\bigg|^{r}\Bigg)^{1/r}\nonumber\\&=\Bigg(\sum_{k=1}^{|\mathcal{Y}|}\bigg|\dfrac{\mathcal{I}[y^n=y_k]}{n}-\bar{\mathrm{p}}(y_k|x_1)-\bigg(\dfrac{\mathcal{I}[\hat{y}^{nt}_1=y_k]}{nt}-\bar{\mathrm{p}}(y_k|x_1)\bigg)\bigg|^{r}\Bigg)^{1/r}\nonumber\\&\myleqa\Bigg(\sum_{k=1}^{|\mathcal{Y}|}\bigg|\dfrac{\mathcal{I}[y^n=y_k]}{n}-\bar{\mathrm{p}}(y_k|x_1)\bigg|^{r}\Bigg)^{1/r}+\Bigg(\sum_{k=1}^{|\mathcal{Y}|}\bigg|\dfrac{\mathcal{I}[\hat{y}^{nt}_1=y_k]}{nt}-\bar{\mathrm{p}}(y_k|x_1)\bigg|^{r}\Bigg)^{1/r},
	\end{align}
	where $(a)$ follows from the Minkowski inequality. 
	Combining \eqref{eq_e18}, \eqref{eq_e19}, and \eqref{eq_e20}, we obtain
	\begin{align}\label{eq_e21}
		\Bigg(\sum_{k=1}^{|\mathcal{Y}|}\bigg|\dfrac{\mathcal{I}[y^n=y_k]}{n}-\dfrac{\mathcal{I}[\hat{y}_1^{nt}=y_k]}{nt}\bigg|^{r}\Bigg)^{1/r}\leq |\mathcal{Y}|^{1/r}\epsilon+|\mathcal{Y}|^{1/r}\frac{\epsilon}{\sqrt[3]{t}}.
	\end{align}
	Hence, the Minkowski distance between the empirical probability distribution of the observed vector $y^n$ and the empirical probability distribution of the concatenated training vector for label $x_1$ is upper bounded by the right hand side of \eqref{eq_e21}. We now derive a lower bound for $\hat{y}_i^{nt}$, where $i\neq1$. For any $x_i$, such that $i\neq 1$, we have
	\begin{align}\label{eq_e22}
	&\Bigg(\sum_{k=1}^{|\mathcal{Y}|}\bigg|\dfrac{\mathcal{I}[y^n=y_k]}{n}-\dfrac{\mathcal{I}[\hat{y}_i^{nt}=y_k]}{nt}\bigg|^{r}\Bigg)^{1/r}+	\Bigg(\sum_{k=1}^{|\mathcal{Y}|}\epsilon^r\Bigg)^{1/r}\nonumber\\
	&\mygeqa\Bigg(\sum_{k=1}^{|\mathcal{Y}|}\bigg|\dfrac{\mathcal{I}[y^n=y_k]}{n}-\dfrac{\mathcal{I}[\hat{y}_i^{nt}=y_k]}{nt}\bigg|^{r}\Bigg)^{1/r}+\Bigg(\sum_{k=1}^{|\mathcal{Y}|}\bigg|\dfrac{\mathcal{I}[y^n=y_k]}{n}-\bar{\mathrm{p}}(y_k|x_1)\bigg|^{r}\Bigg)^{1/r}\nonumber\\
	&\mygeqb \Bigg(\sum_{k=1}^{|\mathcal{Y}|}\bigg|\dfrac{\mathcal{I}[\hat{y}^{nt}_i=y_k]}{nt}-\bar{\mathrm{p}}(y_k|x_1)\bigg|^{r}\Bigg)^{1/r},
	\end{align}
	where $(a)$ follows from \eqref{eq_e18} and $(b)$ is again due to the Minkowski inequality. The expression in \eqref{eq_e22}, can be written equivalently as
	\begin{align}\label{eq_e23}
		&\Bigg(\sum_{k=1}^{|\mathcal{Y}|}\bigg|\dfrac{\mathcal{I}[y^n=y_k]}{n}-\dfrac{\mathcal{I}[\hat{y}_i^{nt}=y_k]}{nt}\bigg|^{r}\Bigg)^{1/r}\nonumber\\&\geq\Bigg(\sum_{k=1}^{|\mathcal{Y}|}\bigg|\dfrac{\mathcal{I}[\hat{y}^{nt}_i=y_k]}{nt}-\bar{\mathrm{p}}(y_k|x_1)\bigg|^{r}\Bigg)^{1/r}-|\mathcal{Y}|^{1/r}\epsilon,
	\end{align}
	where $i\neq1$. Now, using the definitions of $P_{\hat{y}^{nt}_i}$ and $\bar{\mathrm{P}}_1$ given by \eqref{eq_e4} and \eqref{eq_e7}, respectively, into \eqref{eq_e23} we can replace the expression in the right-hand side of \eqref{eq_e23} by $\big\lVert P_{\hat{y}^{nt}_i}-\bar{\mathrm{P}}_1\big\rVert_r$, and thereby for any $i\neq1$ we have
	\begin{align}\label{eq_e24}
		\Bigg(\sum_{k=1}^{|\mathcal{Y}|}\bigg|\dfrac{\mathcal{I}[y^n=y_k]}{n}-\dfrac{\mathcal{I}[\hat{y}_i^{nt}=y_k]}{nt}\bigg|^{r}\Bigg)^{1/r}\geq &\big\lVert P_{\hat{y}^{nt}_i}-\bar{\mathrm{P}}_1\big\rVert_r-|\mathcal{Y}|^{1/r}\epsilon.
	\end{align}
	The expression in \eqref{eq_e24} represents a lower bound on the Minkowski $r$ distance between the empirical probability distribution of the observed vector $y^n$ and the empirical probability distribution of the concatenated training vector for any label $x_i$, where $i\neq1$.
	
	Using the bounds in \eqref{eq_e21} and \eqref{eq_e24}, we now relate the left-hand sides of \eqref{eq_e21} and \eqref{eq_e24} as follows. As long as the following inequality holds for each $i\neq1$,
	\begin{align}\label{eq_e25}
	|\mathcal{Y}|^{1/r}\epsilon\bigg(1+\frac{1}{\sqrt[3]{t}}\bigg)<\lVert P_{\hat{y}^{nt}_i}-\bar{\mathrm{P}}_1\big\rVert_r-|\mathcal{Y}|^{1/r}\epsilon,
	\end{align}
	which is equivalent to the following for $i\neq1$
	\begin{align}\label{eq_e26}
		\epsilon<\dfrac{\big\lVert P_{\hat{y}^{nt}_i}-\bar{\mathrm{P}}_1\big\rVert_r}{(2+t^{-1/3})|\mathcal{Y}|^{1/r}},
	\end{align}
	we have the following for $i\neq1$
	\begin{align}\label{eq_e27}
		\Bigg(\sum_{k=1}^{|\mathcal{Y}|}\bigg|\dfrac{\mathcal{I}[y^n=y_k]}{n}-\dfrac{\mathcal{I}[\hat{y}_1^{nt}=y_k]}{nt}\bigg|^{r}\Bigg)^{1/r}&\myleqa 	|\mathcal{Y}|^{1/r}\epsilon\bigg(1+\frac{1}{\sqrt[3]{t}}\bigg)\nonumber\\
		&\myleqqb \lVert P_{\hat{y}^{nt}_i}-\bar{\mathrm{P}}_1\big\rVert_r-|\mathcal{Y}|^{1/r}\epsilon\nonumber\\&\myleqc 	\Bigg(\sum_{k=1}^{|\mathcal{Y}|}\bigg|\dfrac{\mathcal{I}[y^n=y_k]}{n}-\dfrac{\mathcal{I}[\hat{y}_i^{nt}=y_k]}{nt}\bigg|^{r}\Bigg)^{1/r},
	\end{align}
	where $(a)$, $(b)$, and $(c)$ follow from \eqref{eq_e21}, \eqref{eq_e25}, and \eqref{eq_e24}, respectively. Thereby, from \eqref{eq_e27}, we have the following for $i\neq1$
	\begin{align}\label{eq_e28}
		\Bigg(\sum_{k=1}^{|\mathcal{Y}|}\bigg|\dfrac{\mathcal{I}[y^n=y_k]}{n}-\dfrac{\mathcal{I}[\hat{y}_1^{nT}=y_k]}{nT}\bigg|^{r}\Bigg)^{1/r}\leq\Bigg(\sum_{k=1}^{|\mathcal{Y}|}\bigg|\dfrac{\mathcal{I}[y^n=y_k]}{n}-\dfrac{\mathcal{I}[\hat{y}_i^{nT}=y_k]}{nT}\bigg|^{r}\Bigg)^{1/r}.
	\end{align}
	Note that the right- and left-hand sides of \eqref{eq_e28} can be replaced by the Minkowski distance of the vectors
	\begin{align}\label{kk1}
		  \Bigg[\dfrac{\mathcal{I}\big[y^n=y_1\big]}{n}-\dfrac{\mathcal{I}\big[\hat{y}_1^{nt}=y_1\big]}{nt},\ldots,\dfrac{\mathcal{I}\big[y^n=y_{|\mathcal{Y}|}\big]}{n}-\dfrac{\mathcal{I}\big[\hat{y}^{nt}_1=y_{|\mathcal{Y}|}\big]}{nt}\Bigg] 
	\end{align}
	and
	\begin{align}\label{kk2}
		  \Bigg[\dfrac{\mathcal{I}\big[y^n=y_1\big]}{n}-\dfrac{\mathcal{I}\big[\hat{y}_i^{nt}=y_1\big]}{nt},\ldots,\dfrac{\mathcal{I}\big[y^n=y_{|\mathcal{Y}|}\big]}{n}-\dfrac{\mathcal{I}\big[\hat{y}^{nt}_i=y_{|\mathcal{Y}|}\big]}{nt}\Bigg],
	\end{align}
	respectively. Now, \eqref{kk1} and \eqref{kk2} can be replaced by
	$ P_{y^n}-P_{\hat{y}_1^{nt}}$ and $ P_{y^n}-P_{\hat{y}_i^{nt}}$, respectively, by the definitions of $P_{y^n}$ and $P_{\hat{y}^{nt}_i}$ given by \eqref{eq_e5} and \eqref{eq_e4}, respectively.
	Therefore, \eqref{eq_e28} can be written equivalently as  
	\begin{align}\label{eq_e29}
		\big\lVert P_{y^n}-P_{\hat{y}_1^{nt}}\big\rVert_r<\big\lVert P_{y^n}-P_{\hat{y}_i^{nt}}\big\rVert_r.
	\end{align}
	 Now, let us highlight what we have obtained. We obtained that if there is an $\epsilon$ for which \eqref{eq_e26} holds for $i\neq1$, and for that $\epsilon$ there are sets $\mathcal{A}^{\epsilon}$ and $\mathcal{B}^{\epsilon}$ for which $y^n\in\mathcal{A}^{\epsilon}$ and $\hat{y}^{nt}_1\in\mathcal{B}^{\epsilon}$ then \eqref{eq_e29} holds for $i\neq1$, and thereby our classifier will detect that $x_1$ is the correct label. Using this we can upper bound the error probability as
	\begin{align}\label{eq_e30}
	\mathbb{P}_\e&=1-\pr\big\{\hat{x}_1= x_1\big\}\nonumber\\&\leq 1-\pr\Big\{\big(y^n\in\mathcal{A}^{\epsilon}\big)\cap\big(\hat{y}^{nt}_1\in\mathcal{B}^{\epsilon}\big)\big\lvert\epsilon\in\mathcal{S}\Big\},
	\end{align}
	where $\mathcal{S}$ is a set defined as
	\begin{align}\label{eq_e31}
		\mathcal{S}=\Bigg\{\epsilon:\epsilon\leq\min_{\substack{i\\i\neq 1}}\dfrac{\big\lVert P_{\hat{y}^{nt}_i}-\bar{\mathrm{P}}_1\big\rVert_r}{(2+t^{-1/3})|\mathcal{Y}|^{1/r}}\Bigg\}.
	\end{align}
	In the following, we derive the expression in \eqref{eq_e30}. The right-hand side of \eqref{eq_e30} can be upper bounded as
	\begin{align}\label{eq_e32}
	1-\pr\Big\{\big(y^n\in\mathcal{A}^{\epsilon}\big)\cap\big(\hat{y}^{nt}_1\in\mathcal{B}^{\epsilon}\big)\big\lvert\epsilon\in\mathcal{S}\Big\}&=\pr\Big\{\big(y^n\notin\mathcal{A}^{\epsilon}\big)\cup\big(\hat{y}^{nt}_1\notin\mathcal{B}^{\epsilon}\big)\big\lvert\epsilon\in\mathcal{S}\Big\}\nonumber\\
	&\myleqa\pr\big\{y^n\notin\mathcal{A}^{\epsilon}\lvert\epsilon\in\mathcal{S}\big\}+\pr\Big\{\hat{y}^{nt}_1\notin\mathcal{B}^{\epsilon}\big\lvert\epsilon\in\mathcal{S}\Big\},
	\end{align}
	where $(a)$ follows from Boole's inequality.
	Now, note that we have the following upper bound for the first expression in the right-hand side of \eqref{eq_e32}
	\begin{align}\label{eq_e33}
		\pr\big\{y^n\notin\mathcal{A}^{\epsilon}\lvert\epsilon\in\mathcal{S}\big\}&=\pr\Bigg\{y^n\notin\bigcap\limits_{k=1}^{|\mathcal{Y}|} \mathcal{A}^{\epsilon}_k\bigg\lvert\epsilon\in\mathcal{S}\Bigg\}\nonumber\\
		&=\pr\Bigg\{y^n\in\bigcup\limits_{k=1}^{|\mathcal{Y}|}\overline{\mathcal{A}^{\epsilon}_k}\bigg\lvert\epsilon\in\mathcal{S}\Bigg\}\nonumber\\&\myleqa\sum_{k=1}^{|\mathcal{Y}|}\pr\big\{y^n\in\overline{\mathcal{A}^{\epsilon}_k}\lvert\epsilon\in\mathcal{S}\big\}\nonumber\\&=\sum_{k=1}^{|\mathcal{Y}|}\pr\bigg\{\bigg|\dfrac{\mathcal{I}[y^n=y_k]}{n}-\bar{\mathrm{p}}(y_k|x_1)\bigg|>\epsilon\bigg\lvert\epsilon\in\mathcal{S}\bigg\}\nonumber\\&=\sum_{k=1}^{|\mathcal{Y}|}\pr\Bigg\{\Bigg|\sum_{j=1}^{n}\dfrac{\mathcal{Z}[y_j=y_k]}{n}-\bar{\mathrm{p}}(y_k|x_1)\Bigg|>\epsilon\bigg\lvert\epsilon\in\mathcal{S}\Bigg\},
	\end{align}
	where $\overline{\mathcal{A}^{\epsilon}_k}$ is the complement of $\mathcal{A}^{\epsilon}_k$ and $(a)$ follows from Boole's inequality. Note that $\mathcal{Z}[y_1=y_k],\mathcal{Z}[y_2=y_k],\ldots,\mathcal{Z}[y_n=y_k]$ in \eqref{eq_e33} are $n$ independent Bernoulli random variables with probabilities of success $p_{_1}(y_k|x_1),p_{_2}(y_k|x_1),\ldots,p_{_n}(y_k|x_1)$, respectively. Let $\mathcal{W}[y_k]$ be a binomial random variable with parameters $\big(n,\bar{\mathrm{p}}(y_k|x_1)\big)$. 
	We proceed the proof by introducing the following well-known Hoefdding's Theorem from \cite{hoeffding1956}.
	\begin{theorem}[Hoeffding \cite{hoeffding1956}]\label{th4}
		Assume that $Z_1,Z_2,\ldots,$ and $Z_n$ are $n$ independent Bernoulli random variables with probabilities of success $p_{_1},p_{_2},\ldots,$ and $p_{_n}$, respectively. Next, let $\mathrm{Z}$ be defined as $\mathrm{Z}=Z_1+Z_2+\ldots+Z_n$ and, let $\bar{\mathrm{p}}$ be defined as $\bar{\mathrm{p}}=\big(p_{_1}+p_{_2}+\ldots+p_{_n}\big)/n$. Let $\mathrm{W}$ be a binomial random variable with parameters $(n,\bar{\mathrm{p}})$. Then, for a given $a$ and $b$, where $0\leq a\leq n\bar{\mathrm{p}}\leq b\leq n$ holds, we have
		\begin{align}\label{eq_e13}
		\pr\big\{a\leq \mathrm{W}\leq b\big\}\leq\pr\big\{a\leq \mathrm{Z}\leq b\big\}.
		\end{align}
		In other words, the probability distribution of $\mathrm{W}$ is more dispersed around its mean $n\bar{\mathrm{p}}$ than is the probability distribution of $\mathrm{Z}$. 	Except in the trivial case when $a=b=0$, the bound in \eqref{eq_e13} holds with equality if and only if $p_1=\ldots=p_n=\bar{\mathrm{p}}$. 
	\end{theorem}
	\begin{proof}
		Please refer to \cite{hoeffding1956}.
	\end{proof}
	Setting $a=n(\bar{\mathrm{p}}-\delta)$ and $b=n(\bar{\mathrm{p}}+\delta)$ in \eqref{eq_e13}, we obtain
	\begin{align}\label{eq_e14}
	\pr\big\{n(\bar{\mathrm{p}}-\delta)\leq \mathrm{W}\leq n(\bar{\mathrm{p}}+\delta)\big\}\leq\pr\big\{n(\bar{\mathrm{p}}-\delta)\leq \mathrm{Z}\leq n(\bar{\mathrm{p}}+\delta)\big\}.
	\end{align}
	Using \eqref{eq_e14}, we have the following upper bound
	\begin{align}\label{eq_e15}
	\pr\bigg\{\bigg|\dfrac{\mathrm{Z}}{n}-\bar{\mathrm{p}}\bigg|>\delta\bigg\}&=1-\pr\big\{n(\bar{\mathrm{p}}-\delta)\leq \mathrm{Z}\leq n(\bar{\mathrm{p}}+\delta)\big\}\nonumber\\&\myleqa 1-\pr\big\{n(\bar{\mathrm{p}}-\delta)\leq \mathrm{W}\leq n(\bar{\mathrm{p}}+\delta)\big\}\nonumber\\&=\pr\bigg\{\bigg|\dfrac{\mathrm{W}}{n}-\bar{\mathrm{p}}\bigg|>\delta\bigg\},
	\end{align}
	where $(a)$ follows from \eqref{eq_e14}.
	
We now turn to the proof of Theorem \ref{th2}. According to Theorem \ref{th4}, the probability distribution of $\mathcal{W}[y_k]$ is more dispersed around its mean $n\bar{\mathrm{p}}(y_k|x_1)$ than is the probability distribution of $\sum_{1\leq j\leq n}\mathcal{Z}[y_j=y_k]$. Therefore, we can upper bound the probability in the last line of \eqref{eq_e33} as
	\begin{align}\label{eq_e34}
		\pr\Bigg\{\Bigg|\sum_{j=1}^{n}\dfrac{\mathcal{Z}[y_j=y_k]}{n}-\bar{\mathrm{p}}(y_k|x_1)\Bigg|>\epsilon\bigg\lvert\epsilon\in\mathcal{S}\Bigg\}\myleqa\pr\bigg\{\bigg|\dfrac{\mathcal{W}[y_k]}{n}-\bar{\mathrm{p}}(y_k|x_1)\bigg|>\epsilon\bigg\lvert\epsilon\in\mathcal{S}\bigg\},
	\end{align}
where $\epsilon\in\mathcal{S}$, defined in \eqref{eq_e31} and $(a)$ follows from \eqref{eq_e15}. Now, let us introduce another well-known Hoeffding's Theorem from \cite{10500830}.
	\begin{theorem}[Hoeffding's inequality \cite{10500830}]\label{th3}
		Let $W_1,W_2,\ldots,W_n$ be $n$ independent random variables such that for each ${1\leq i\leq n}$, we have $\pr\big\{W_i\in[a_i,b_i]\big\}=1$.
		Then for $S_n$, defined as $S_n=\sum\limits_{i=1}^nW_i$, we have
		\begin{align}\label{eq_e11}
		\pr\Big\{S_n-\mathbb{E}\big[S_n\big]\geq\delta\Big\}\leq\exp\Bigg(-\frac{2\delta^2}{\sum_{i=1}^n(b_i-a_i)^2}\Bigg),
		\end{align}
		where $\mathbb{E}\big[S_n\big]$ is the expectation of $S_n$.
	\end{theorem}
	\begin{proof}
		Please refer to \cite{10500830}.
	\end{proof}
	
	Back to \eqref{eq_e34}, by using the result of \eqref{eq_e11} for $a_i=0$ and $b_i=1$ since the binomial random variable $\mathcal{W}[y_k]$ can take values $0$ or $1$, respectively, we have
	
	\begin{align}\label{eq_ekkk}
	    \pr\Bigg\{\Bigg|\sum_{j=1}^{n}\dfrac{\mathcal{Z}[y_j=y_k]}{n}-\bar{\mathrm{p}}(y_k|x_1)\Bigg|>\epsilon\bigg\lvert\epsilon\in\mathcal{S}\Bigg\}&\leq2\exp\Bigg(-\frac{2n^2\epsilon^2}{\sum_{1\leq i\leq n}(1-0)^2}\Bigg)\nonumber\\
		&\leq 2\e^{-2n\epsilon^2},
	\end{align}
	where $\epsilon\in\mathcal{S}$, defined in \eqref{eq_e31}. Inserting \eqref{eq_ekkk} into \eqref{eq_e33}, we obtain the following upper bound
	\begin{align}\label{eq_e35}
		\pr\big\{y^n\notin\mathcal{A}^{\epsilon}\lvert\epsilon\in\mathcal{S}\big\}\leq 2|\mathcal{Y}|\e^{-2n\epsilon^2}.
	\end{align}
Similarly, we have the following result for the second expression in the right-hand side of \eqref{eq_e32}
\begin{align}\label{eq_e36}
\pr\Big\{\hat{y}^{nt}_1\notin\mathcal{B}^{\epsilon}\big\lvert\epsilon\in\mathcal{S}\Big\}&=\pr\Bigg\{\hat{y}^{nt}_1\notin\bigcap\limits_{k=1}^{|\mathcal{Y}|} \mathcal{B}^{\epsilon}_k\bigg\lvert\epsilon\in\mathcal{S}\Bigg\}\nonumber\\
&=\pr\Bigg\{\hat{y}^{nt}_1\in\bigcup\limits_{k=1}^{|\mathcal{Y}|}\overline{\mathcal{B}^{\epsilon}_k}\bigg\lvert\epsilon\in\mathcal{S}\Bigg\}\nonumber\\&\myleqa\sum_{k=1}^{|\mathcal{Y}|}\pr\big\{\hat{y}^{nt}_1\in\overline{\mathcal{B}^{\epsilon}_k}\lvert\epsilon\in\mathcal{S}\big\}\nonumber\\&=\sum_{k=1}^{|\mathcal{Y}|}\pr\bigg\{\bigg|\dfrac{\mathcal{I}[\hat{y}^{nt}_1=y_k]}{nt}-\bar{\mathrm{p}}(y_k|x_1)\bigg|>\frac{\epsilon}{\sqrt[3]{t}}\bigg\lvert\epsilon\in\mathcal{S}\bigg\}\nonumber\\&=\sum_{k=1}^{|\mathcal{Y}|}\pr\Bigg\{\Bigg|\sum_{j=1}^{nt}\dfrac{\mathcal{Z}[y_j=y_k]}{nt}-\bar{\mathrm{p}}(y_k|x_1)\Bigg|>\frac{\epsilon}{\sqrt[3]{t}}\bigg\lvert\epsilon\in\mathcal{S}\Bigg\},
\end{align}
where again $(a)$ follows from Boole's inequality. Note that due to \eqref{eq_e3}, for any integer number $l$ such that $0\leq l\leq t-1$ the random variables $\mathcal{Z}[y_{nl+1}=y_k],\mathcal{Z}[y_{nl+2}=y_k],\ldots,$ and $\mathcal{Z}[y_{nl+n}=y_k]$ in \eqref{eq_e36} are $n$ independent Bernoulli random variables with the probabilities of success $p_{_1}(y_k|x_1),p_{_2}(y_k|x_1),\ldots,$ and $p_{_n}(y_k|x_1)$, respectively $\big(y_{nl+1},y_{nl+2},\ldots,y_{nl+n}$ are elements of $\hat{y}^n_{1_{l+1}}\big)$. Also, note that
	\begin{align}\label{eq_e37}
		\bar{\mathrm{p}}(y_k|x_1)&=\dfrac{1}{n}\sum_{j=1}^{n}p_{_j}(y_k|x_1)\nonumber\\&=\dfrac{1}{nt}\Bigg(\sum_{l=0}^{t-1}\sum_{j=1}^{n}p_{_j}(y_k|x_1)\Bigg).
	\end{align}
	Notice that for each $0\leq l\leq t-1$, $p_{_1}(y_k|x_1)+p_{_2}(y_k|x_1)+\ldots+p_{_n}(y_k|x_1)$ is the summation of the probabilities of success of the random variables $\mathcal{Z}[y_{nl+1}=y_k],\mathcal{Z}[y_{nl+2}=y_k],\ldots,$ and $\mathcal{Z}[y_{nl+n}=y_k]$. Thereby, the last expression in right-hand side of \eqref{eq_e37} is the average probability of success of random variables $\mathcal{Z}[y_j=y_k]$ for $1\leq j\leq nt$. Now, let $\mathcal{W}[y_k]$ be a binomial random variable with parameters $\big(nt,\bar{\mathrm{p}}(y_k|x_1)\big)$. Once again, according to Theorem \ref{th4}, the probability distribution of $\mathcal{W}[y_k]$ is more dispersed around its mean $nt\bar{\mathrm{p}}(y_k|x_1))$ than is the probability distribution of $\sum_{1\leq j\leq nt}\mathcal{Z}[y_j=y_k]$. Therefore, the probability in the last line of \eqref{eq_e36} can be upper bounded as
	\begin{align}\label{eq_e38}
		\pr\Bigg\{\Bigg|\sum_{j=1}^{nt}\dfrac{\mathcal{Z}[y_j=y_k]}{nt}-\bar{\mathrm{p}}(y_k|x_1)\Bigg|>\frac{\epsilon}{\sqrt[3]{t}}\bigg\lvert\epsilon\in\mathcal{S}\Bigg\}&\myleqa\pr\bigg\{\bigg|\dfrac{\mathcal{W}[y_k]}{nt}-\bar{\mathrm{p}}(y_k|x_1)\bigg|>\frac{\epsilon}{\sqrt[3]{t}}\bigg\lvert\epsilon\in\mathcal{S}\bigg\}\nonumber\\
		&\myleqb2\exp\Bigg(-\frac{2(nt)^2\big(t^{-1/3}\epsilon\big)^2}{\sum_{1\leq i\leq nt}(1-0)^2}\Bigg)\nonumber\\
		&\leq 2\e^{-2nt\big(t^{-2/3}\epsilon^2\big)}\nonumber\\
		&=2\e^{-2nt^{1/3}\epsilon^2},
	\end{align}
	where $\epsilon\in\mathcal{S}$, defined in \eqref{eq_e31}, $(a)$ follows from \eqref{eq_e15} (in which $n$ is replaced by $nt$), and $(b)$ is the result of \eqref{eq_e11} for $a_i=0$ and $b_i=1$ since the binomial random variable $\mathcal{W}[y_k]$ can take values $0$ or $1$, respectively. Inserting \eqref{eq_e38} into \eqref{eq_e36}, we have the following upper bound 
	\begin{align}\label{eq_e39}
	\pr\Big\{\hat{y}^{nt}_1\notin\mathcal{B}^{\epsilon}\big\lvert\epsilon\in\mathcal{S}\Big\}\leq 2|\mathcal{Y}|\e^{-2nt^{1/3}\epsilon^2}.
	\end{align}
	Inserting \eqref{eq_e35} and \eqref{eq_e39} into \eqref{eq_e32}, and then inserting \eqref{eq_e32} into \eqref{eq_e30}, we obtain the following upper bound for the error probability
	\begin{align}\label{eq_e40}
	\mathbb{P}_\e\leq 2|\mathcal{Y}|\e^{-2n\epsilon^2}+2|\mathcal{Y}|\e^{-2nt^{1/3}\epsilon^2},
	\end{align}
	where 
	\begin{align}\label{eq_e41}
		\epsilon = \min_{\substack{i,j\\i\neq j}}\dfrac{\big\lVert P_{\hat{y}^{nt}_i}-\bar{\mathrm{P}}_j\big\rVert_r}{(2+t^{-1/3})|\mathcal{Y}|^{1/r}},
	\end{align}
	which is the optimal value of $\epsilon$ that exhibits the tightest upper bound for the error probability $\mathbb{P}_\e$ given by \eqref{eq_e40}. This completes the proof of Theorem \ref{th2}.
\end{proof}

The following corollary provides a simplified upper bound on the error probability when $t\to\infty$.

\begin{cor}
	When the number of training vectors per label goes to infinity, i.e., when $t\to\infty$, which is equivalently to the case when the probability distribution $p(y^n|x)$ is known at the classifier, the error probability of the proposed classifier is upper bounded as
	\begin{align}\label{eq_e42}
		\mathbb{P}_\e\leq 2|\mathcal{Y}|\e^{-2n\epsilon^2},
	\end{align}
	where $\epsilon$ is given by
	\begin{align}\label{eq_e43}
		\epsilon = \min_{\substack{i,j\\i\neq j}}\dfrac{\big\lVert \bar{\mathrm{P}}_i-\bar{\mathrm{P}}_j\big\rVert_r}{2|\mathcal{Y}|^{1/r}}.
	\end{align}
\end{cor}
\begin{proof}
The proof is straightforward.
\end{proof}

As can be seen from \eqref{eq_e6} and \eqref{eq_e9}, the performance of the proposed classifier depends on   $r$. We cannot derive the optimal value of $r$ that minimizes the error probability since we do not have the exact expression of the error probability, we only have its upper bound. On the other hand, in practice,   the optimal $r$ with respect to the upper bound on the error probability also cannot be derived since the  upper bound depends on $\bar{\mathrm{P}}_j$, which would be unknown in practice due to $p_{Y^n|X}(y^n|x)$ being unknown.  As a result, for our numerical examples, we consider the Euclidean distance $(r=2)$, which is one of the most widely used distance metric in practice.  
 
The following corollary establishes the asymptotic optimality of the proposed classifier with respect to $n$.

\begin{cor}
The proposed classifier has an error probability that satisfies $\mathbb{P}_\e\to0$ as $n\to\infty$ if $|\mathcal{Y}|\leq\mathcal{O}(n^m)$, $m$ is fixed, and $r>2m$, where $m$ is an arbitrary fixed integer. Here, $n^m$ indicates the dimension of our space, i.e., maximum number of alphabets each element in the data vector $y^n$ can take. Thereby, the proposed classifier is asymptotically optimal .
\end{cor}
\begin{proof}
For the proof, please see Appendix \ref{A1}.
\end{proof}

\section{Simulation Results}\label{sec4}
In this section, we provide simulation results of the performance of the proposed classifier for $r=2$ and compare it to benchmark schemes. The benchmark schemes that we adopt for comparison are the naive Bayes classifier and the KNN algorithm. 
We cannot adopt a classifier based on a neural network since neural networks require a very large training set, which we assume it is not available. For the naive Bayes classifier, the probability distribution $p_{_{Y^n|X}}(y^n|x)$ is estimated from the training vectors as follows. Let again $\hat{y}^{nt}_i$ be a vector obtained by concatenating all training feature vectors for the input label $x_i$ as in \eqref{eq_e3}. Then, the estimated probability distribution of $p(y_j=y|x_i)$, denoted by $\hat{p}(y_j=y|x_i)$, is found as
\begin{align}\label{eq_e44}
\hat{p}(y_j=y|x_i)=\dfrac{\mathcal{I}\big[\hat{y}^{nt}_i=y\big]}{nt},
\end{align}
and the naive Bayes classifier decides according to
\begin{align}\label{eq_e45}
\hat{x} = \arg\max_{x_i}\prod_{k=1}^n\hat{p}(y_k|x_i).
\end{align}
The main problem of the naive Bayes classifier occurs when an alphabet $y_j\in\mathcal{Y}$ is not present in the training feature vectors. In that case, $\hat{p}(y_j|x_i)$   in \eqref{eq_e44}  is $\hat{p}(y_j|x_i)=0$, $\forall x_i\in\mathcal{X}$, and as a result  the right hand side of \eqref{eq_e45} is zero since at least one of the elements in the product in \eqref{eq_e45} is zero. In this case, the naive Bayes classifier fails to provide an accurate classification of the labels. In what follows, we see that this issue of the naive Bayes classifier appears frequently when we have a small number of training feature vectors. On the other hand, the KNN classifier works  as follows. For the observed feature vector $y^n$, the KNN classifier looks for the $k$ nearest feature vectors to $y^n$, among all training feature vectors $\hat{y}_{r_s}^n$, for all $1\leq r\leq |\mathcal{X}|$ and $1\leq s\leq T$. Then by considering a set of $K$ input-output pairs $(x_k,\hat{y}^n_{k_l})$, for $k\in\{1,2,\ldots,|\mathcal{X}|\}$ and $l\in\{1,2,\ldots,|T|\}$, the KNN classifier decides a label which is most frequent among  $x_k$-s. The optimum value of $k$ for $t=1$ is $k=1$.

In the following, we provide numerical examples where we illustrate the performance of the proposed classifier when $p_{_{Y^n|X}}(y^n|x)$ is artificially generated.  

\subsection{The I.I.D. Case With One Training Sample Per Label}
In the following examples, we assume that the classifiers have access to only one training feature vector for each label, the elements of the feature vectors are generated i.i.d., and the   alphabet  size of the feature vector, $|\mathcal{Y}|$, is fixed.

\begin{figure} 
	\centering			\includegraphics[width=1\linewidth]{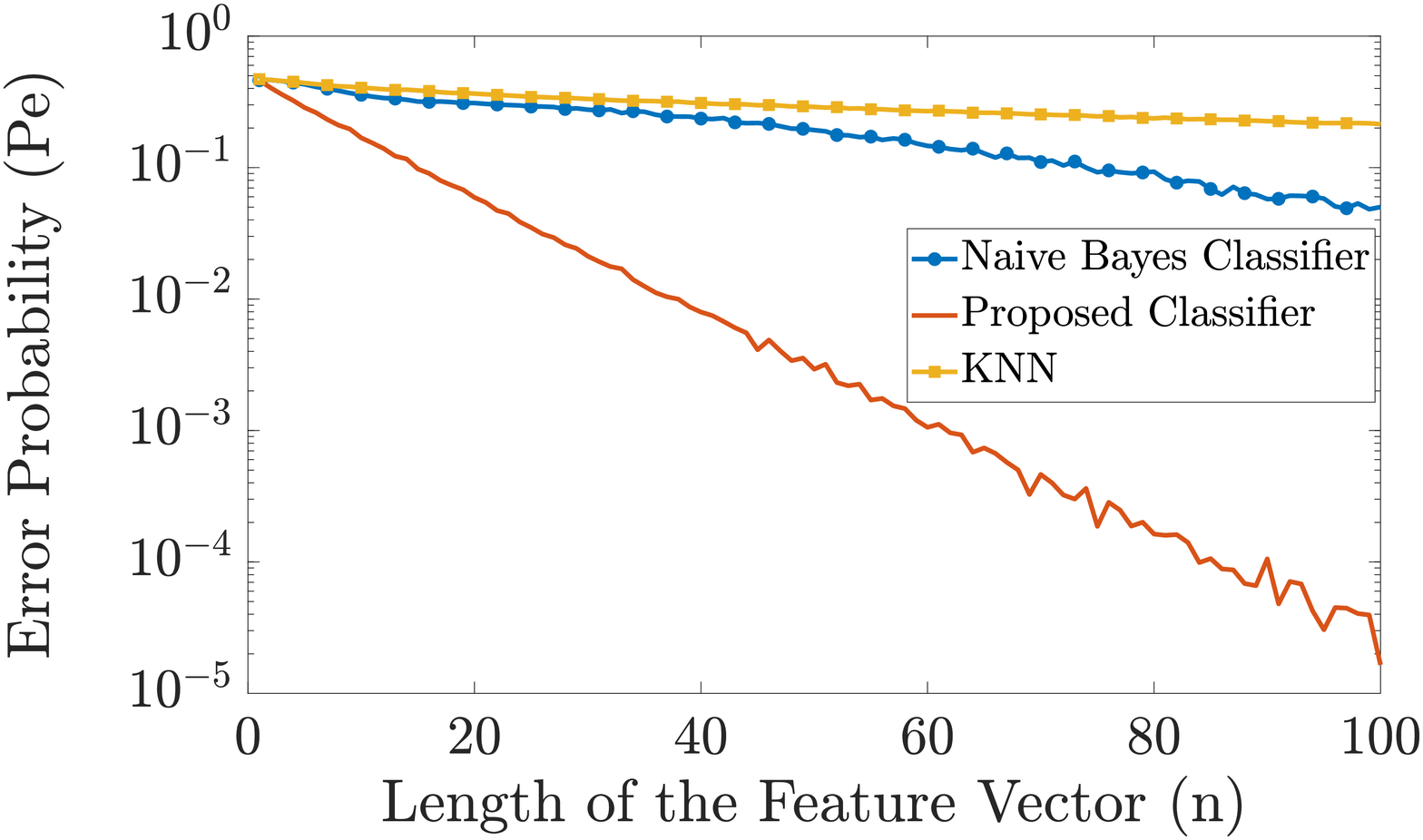}
	\caption{Comparison in error probability between the naive Bayes classifier, KNN, and the proposed classifier.} \label{fig:fig5}
\end{figure}

\begin{figure} 
	\centering\includegraphics[width=1\linewidth]{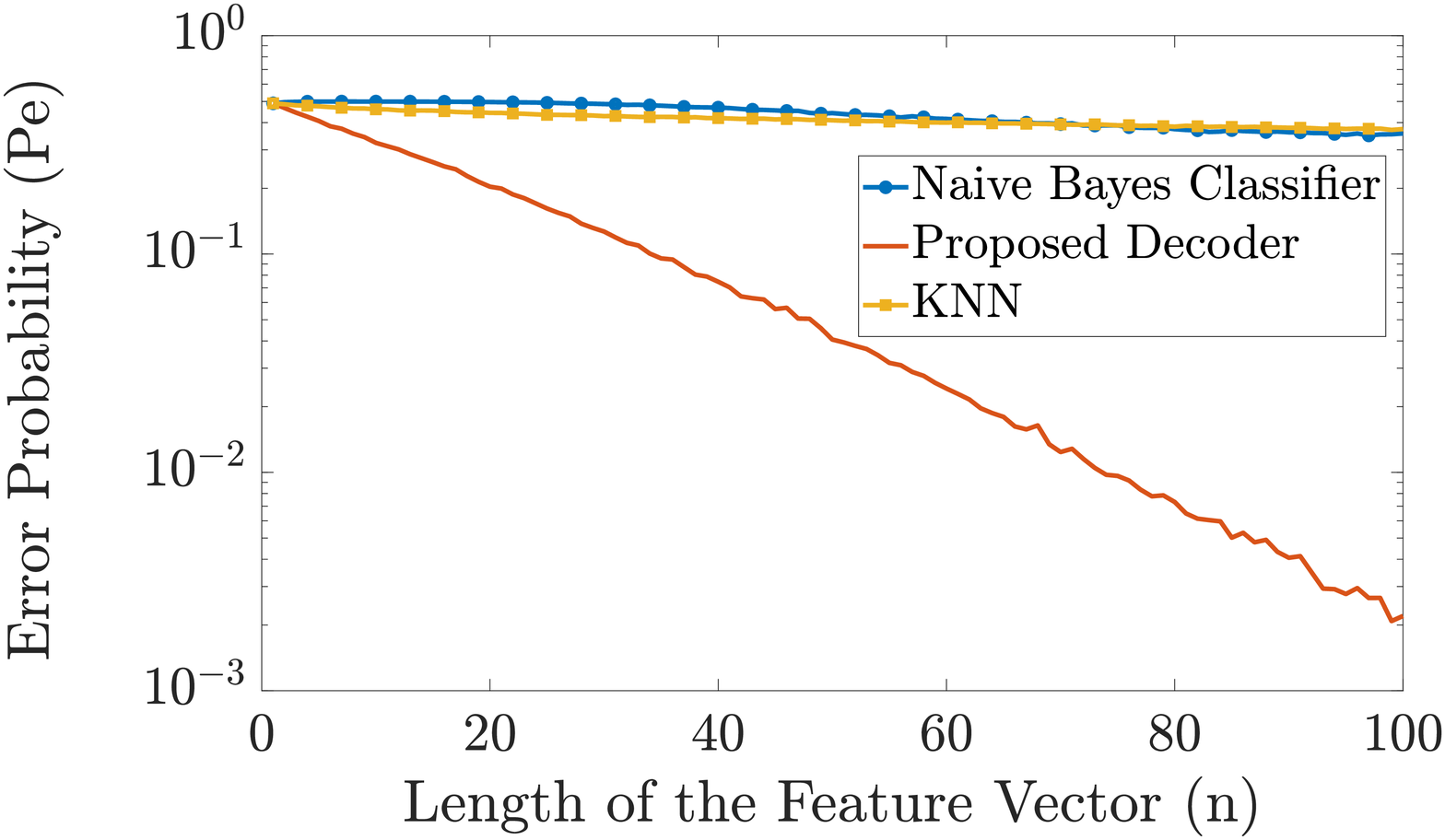}
	\caption{Comparison in error probability between the naive Bayes classifier, KNN, and the proposed classifier.} \label{fig:fig1}
\end{figure}

\begin{figure}[!h]
	\centering			\includegraphics[width=1\linewidth]{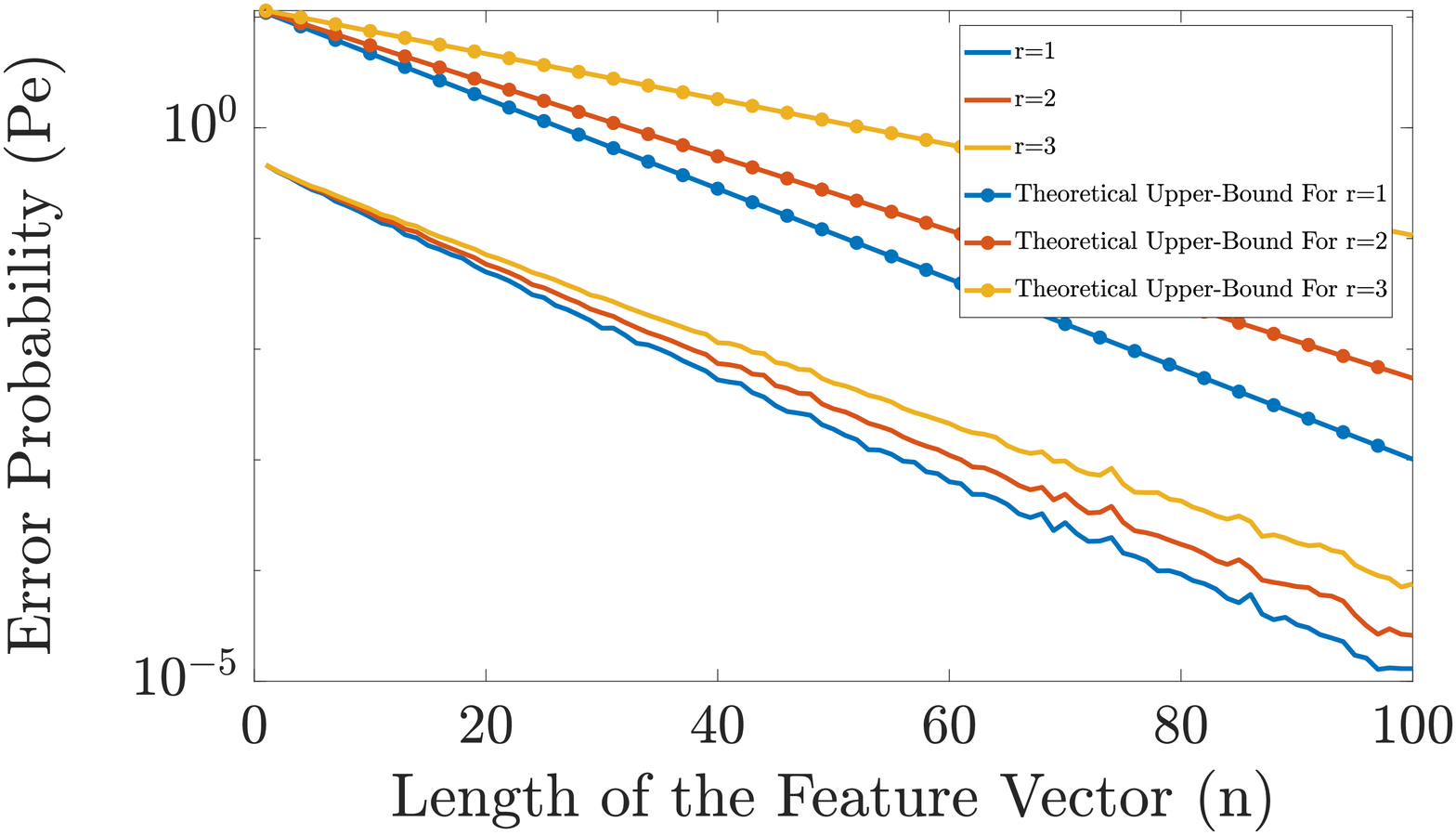}
	\caption{Comparison in error probability of the proposed classifier for different values of $r$ when $|\mathcal{Y}|=6$. The related theoretical upper bounds for each value of $r$ are also given.} \label{fig:fig03}
\end{figure}

In Figs. \ref{fig:fig5} and  \ref{fig:fig1}, we compare the error probability of the proposed classifier with the naive Bayes classifier and the KNN algorithm for the case when $|\mathcal{Y}|=6$ and $|\mathcal{Y}|=20$, respectively. In both examples, we have two different labels, i.e., $|\mathcal{X}|=2$. As a result, we have two different probability distributions $p_{_{Y^n|X_1}}(y^n|x_1)$ and $p_{_{Y^n|X_2}}(y^n|x_2)$. The probability distributions $p_{_{Y^n|X_1}}(y^n|x_1)$ and $p_{_{Y^n|X_2}}(y^n|x_2)$ are randomly generated as follows. We first generate two random vectors of length 6 and length 20 for Figs. \ref{fig:fig5} and  \ref{fig:fig1}, respectively, where the elements of these vectors are drawn independently from a uniform probability distribution. Then we normalize these vectors such that the sum of their elements is equal to one. These two normalized randomly generated vectors then represent the two probability distributions $p_{_{Y_i|X_1}}(y_i|x_1)=p_{_{Y|X_1}}(y|x_1)$ and $p_{_{Y_i|X_2}}(y_i|x_2)=p_{_{Y |X_2}}(y |x_2)$, $\forall i$. Then, $p_{_{Y^n|X_k}}(y^n|x_k)$ is obtained as $p_{_{Y^n|X_k}}(y^n|x_k)= \prod_{i=1}^n p_{_{Y_i|X_k}} (y_i|x_k)$, for $k=1,2$. The simulation is carried out as follows. For each $n$, we generate one training vector for each label, using the aforementioned probability distributions. Then, as test samples, we generate $1000$ feature vectors for each label and pass these feature vectors through our proposed classifier, the naive Bayes classifier, and the KNN algorithm, and compute the errors. The length of the feature vector $n$ is varied from $n=1$ to $n=100$. We repeat the simulation $5000$ times  and then plot the error probability.
Figs. \ref{fig:fig5} and  \ref{fig:fig1} show that   the proposed classifier outperforms both the naive Bayes classification and KNN. The main reason for this performance gain is  because when only one training vector per label is available, the proposed classifier is more resilient to errors than the naive Bayes classifier, whereas the KNN algorithm has very poor performance because of the ``curse of dimensionality''. Specifically, the naive Bayes classifier cannot perform an accurate classification for small $n$ compared to $|\mathcal Y|$ since the chance that  an alphabet will not be present in one of the training feature vectors is close to $1$. On the other hand, the KNN algorithm cannot perform an accurate classification for large $n$ since the dimension of the input feature vector becomes much larger than the training data and the ``curse of dimensionality'' occurs.

In Fig.~\ref{fig:fig03}, we compare the performance of the proposed classifier for different values of $r$ when $|\mathcal{Y}|=6$ with the derived upper bounds.
 As can be seen, for this example, the derived theoretical  upper bounds have similar slope as the exact error probabilities. Moreover, we can  see that for this example, 
  the optimal $r$   is $r=1$. However, this is not always the case and  it depends on $p_{_{Y^n|X_k}}(y^n|x_k)$, $|\mathcal{Y}|$, and $|\mathcal{X}|$.

%



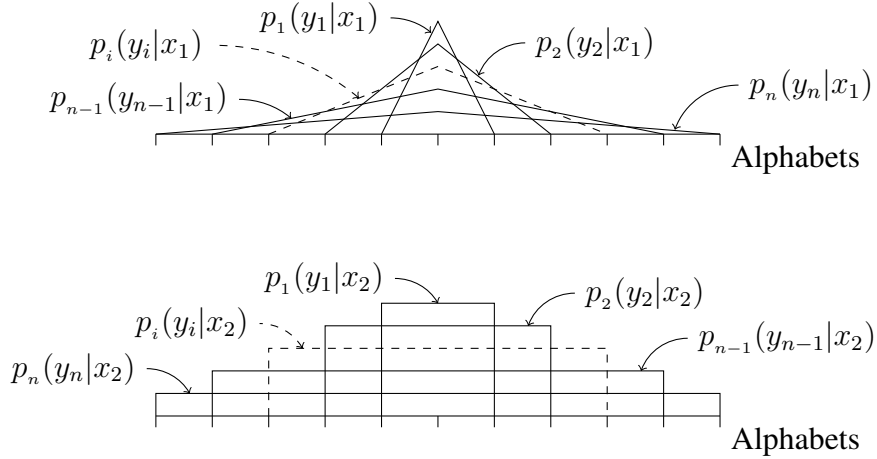
\begin{figure}[!h]
	\centering
	\begin{tikzpicture}[scale=3]
	\draw (-1.25,0) -- (1.25,0) coordinate (x axis) node[below right] {Alphabets};
	\foreach \x in {-1.25,-1,...,1.25}
	\draw (\x,0) -- (\x,-0.05);
	\draw (-1.25,0) -- (0,0.1) -- (1.25,0);
	\draw[<-] (1.05,0.02) arc (160:90:0.3) node[right] {$p_{_n}(y_n|x_1)$};
	\draw (-1,0) -- (0,0.2) -- (1,0);
	\draw[<-] (-0.65,0.07) arc (50:90:0.36) node[left] {$p_{_{n-1}}(y_{n-1}|x_1)$};
	\draw [dashed] (-0.75,0) -- (0,0.3) -- (0.75,0);
	\draw[<-, dashed] (-0.35,0.16) arc (50:90:1) node[left]
	{$p_{_i}(y_i|x_1)$};
	\draw (-0.5,0) -- (0,0.4) -- (0.5,0);
	\draw[<-] (0.18,0.26) arc (160:90:0.22) node[right] {$p_{_2}(y_2|x_1)$};
	\draw (-0.25,0) -- (0,0.5) -- (0.25,0);
	\draw[<-] (-0.03,0.45) arc (50:90:0.28) node[left] {$p_{_1}(y_1|x_1)$};
	
	\draw (-1.25,-1.25) -- (1.25,-1.25) node[below right]
	{Alphabets};
	\foreach \x in {-1.25,-1,...,1.25}
	\draw (\x,-1.25) -- (\x,-1.3);
	\draw (-1.25,-1.25) -- (-1.25,-1.15) -- (1.25,-1.15) -- (1.25,-1.25);
	\draw (-1,-1.25) -- (-1,-1.05) -- (1,-1.05) -- (1,-1.25);
	\draw [dashed] (-0.75,-1.25) -- (-0.75,-0.95) -- (0.75,-0.95) -- (0.75,-1.25);
	\draw (-0.5,-1.25) -- (-0.5,-0.85) -- (0.5,-0.85) -- (0.5,-1.25);
	\draw (-0.25,-1.25) -- (-0.25,-0.75) -- (0.25,-0.75) -- (0.25,-1.25);
	\draw[<-] (0,-0.75) arc (30:90:0.22) node[left] {$p_{_1}(y_1|x_2)$};
	\draw[<-] (0.4,-0.85) arc (160:90:0.22) node[right] {$p_{_2}(y_2|x_2)$};
	\draw[<-, dashed] (-0.6,-0.95) arc (30:90:0.22) node[left]
	{$p_{_i}(y_i|x_2)$};
	\draw[<-] (0.9,-1.05) arc (160:90:0.22) node[right] {$p_{_{n-1}}(y_{n-1}|x_2)$};
	\draw[<-] (-1.1,-1.15) arc (30:90:0.22) node[left]
	{$p_{_n}(y_n|x_2)$};
	\end{tikzpicture}
	\caption[hi]
	{Illustration of the probability distributions $p_{_i}(y_i|x_1)$ (upper figure) and $p_{_i}(y_i|x_2)$ (lower figure), for $i=1,2,\ldots,n$.}
	\label{fig:4.2}
\end{figure}

\subsection{The Overlapping I.Non-I.D. Case With One Training Sample Per Label}\label{koscher}
In this example, we consider the i.non-i.d. case where the probability distributions $p_{_i}(y_i|x_k)$ are overlapping for all $i$, as shown in Fig. \ref{fig:4.2}. The small orthogonal lines on the x-axis in Fig. \ref{fig:4.2} represent alphabets, i.e., the elements in $\mathcal Y$, and the probability of occurrence of an alphabet $y_i$ is equal to the intersection between the corresponding orthogonal line to the represented probability distribution $p_{_i}(y_i|x_k)$ for $k=1,2$. By ``overlapping'', we mean the following. Let $\mathcal{Y}_v$ and $\mathcal{Y}_u$ denote the set of outputs generated by $p_{_v}(y_v|x_k)$ and $p_{_u}(y_u|x_k)$, respectively. If for any $v$ and $u$, $\mathcal{Y}_v\cap\mathcal{Y}_u\neq\emptyset
$ holds, we say that the output alphabets are overlapping. 

To demonstrate the performance of our proposed classifier in the overlapping case, we assume that we have two different labels, $\mathcal{X}=\{x_1,x_2\}$, where the corresponding conditional probability distributions $p_{_i}(y_i|x_1)$ and $p_{_i}(y_i|x_2)$ are obtained as follows. For a given $n$, let $\mathcal{Y}=\big\{-n,-n+1,\ldots,0,\ldots,n-1,n\big\}$ be the set of all alphabets. Note  that the size of $\mathcal Y$ grows with $n$. Also, let $\mathbf{u}_i$ and $\mathbf{v}_i$ $(1\leq i\leq n)$ be vectors of length $2n+1$, given by
\begin{align}\label{eq_e46,eq_e47}
\mathbf{u}_i &= \bigg[0,\ldots,0,\dfrac{1}{i(i+1)},\dfrac{2}{i(i+1)},\ldots,\dfrac{i}{i(i+1)},\dfrac{i+1}{i(i+1)},\dfrac{i}{i(i+1)},\ldots,\dfrac{1}{i(i+1)},0,\ldots,0\bigg],\\
\mathbf{v}_i &= \bigg[0,\ldots,0,\dfrac{1}{i(i+1)},\dfrac{1}{i(i+1)},\ldots,\dfrac{1}{i(i+1)},\dfrac{1}{i(i+1)},0,\ldots,0\bigg].
\end{align}
The number of zeros in each side of the vectors $\mathbf{u}_i$ and $\mathbf{v}_i$ is $(n-i)$. To generate a feature vector from label $x_1(x_2)$, we generate the vector $y^n=(y_1,y_2,\ldots,y_n)$, where $y_k$ takes values from the set $\mathcal{Y}$, with a probability distribution $p_{_i}(y_i|x_1)=\mathbf{u}_i\big(1+2(n+y_i)\big)$ $\Big(p_{_i}(y_i|x_2)=\mathbf{v}_i\big(1+2(n+y_i)\big)\Big)$.

The simulation is carried out as follows. For each $n$, we generate one training feature vector for each label. Then, we generate $1000$ feature vectors for each label and pass them through our proposed classifier, the naive Bayes classifier, and the KNN algorithm  and calculate the error probability. We change the length of the feature vector from $n=1$ to $n=100$ and repeat the simulation $1000$ times and then plot the error probability.
\begin{figure}[!h]
	\centering			\includegraphics[width=1\linewidth]{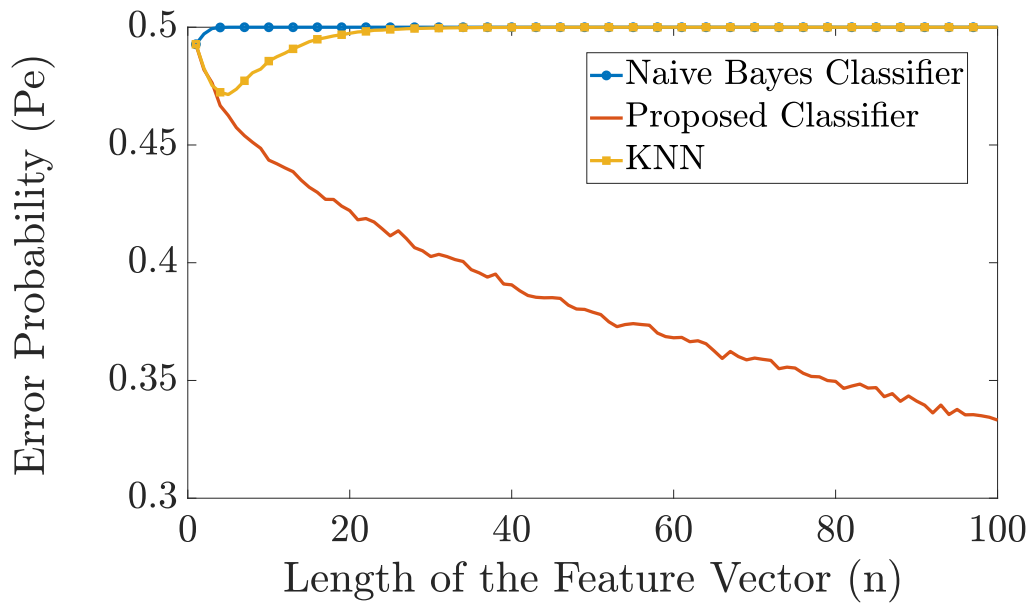}
	\caption{Comparison in error probability between the naive Bayes classifier, KNN, and the proposed classifier $(T=1)$.} \label{fig:fig2}
\end{figure}
As shown in Fig. \ref{fig:fig2}, there is a huge difference between the performance of the two benchmark classifiers and the proposed classifier. The error probability of the naive Bayes classifier is almost $0.5$ for all shown values of $n$ as it is susceptible to the problem of unseen alphabets in the training vectors. The error probability of the KNN classifier is also almost $0.5$ for   $n>20$ as it is susceptible to the ``curse of dimensionality''.
 However, the error probability of our proposed classifier continuously decays as $n$ increases.
 
In Fig.~\ref{fig:fig3}, we run the same experiments as in  Fig.~\ref{fig:fig2} but with $T=100$, i.e., 100 training feature vectors per label.   As   can bee seen from Fig.~\ref{fig:fig3}, the performance of the proposed classifier is better than the naive Bayes classifier, for $n>15$. Since $|\mathcal{Y}|=2n+1$, for small values of $n$, the naive Bayes classifier has access to many training samples, and thereby, it performs very close to the case when the probability distribution $p_{_{Y^n|X}}(y^n|x)$ is known, i.e., to the maximum-likelihood classifier, and hence it has the optimal performance. As $n$ increases, the number of alphabets rises, i.e., $|\mathcal{Y}|$ rises, and due to the aforementioned issue of the naive Bayes classifier with unseen alphabets, our proposed classifier makes much better classification than the naive Bayes classifier. Also, note that the error probability of our proposed classifier decays exponentially as $n$ increases which is not the case with the naive Bayes classifier. Moreover, Fig. \ref{fig:fig3} also shows the   theoretical upper bound on the error probability we derived in \eqref{eq_e9}.
\begin{figure}[!h]
	\centering			\includegraphics[width=1\linewidth]{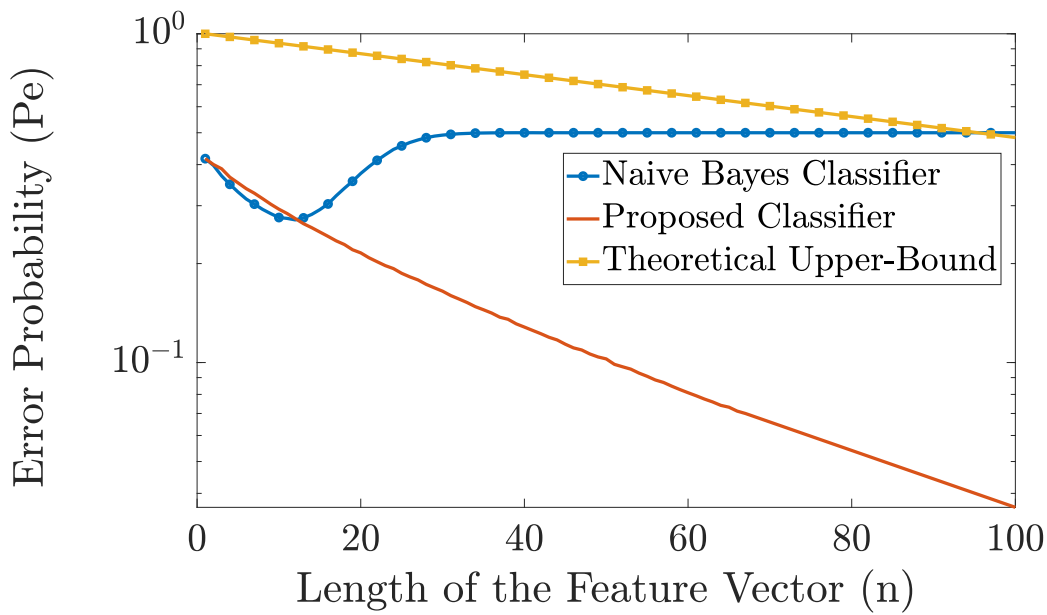}
	\caption{Comparison in error probability between the naive Bayes classifier and the proposed classifier $(T=100)$.} \label{fig:fig3}
\end{figure}
\subsection{The Non-Overlapping I.Non-I.D. Case With One Training Sample For Each Label}
In this example, we consider the i.non-i.d. case where the probability distributions $p_{_j}(y_j|x_i)$ are non-overlapping for all $j$ as shown in Fig. \ref{fig:3.2}, where we defined "overlapping" in Subsection \ref{koscher}. Hence, we test the other extreme in terms of possible distribution of the elements in the feature vectors $Y^n$.

To demonstrate the performance of our proposed classifier in the non-overlapping case, we assume that we have two different labels $\mathcal{X}=\{x_1, x_2\}$, the corresponding conditional probability distributions $p_{_i}(y_i|x_1)$ and $p_{_i}(y_i|x_2)$ are obtained as follows. For a given $n$, let $\mathcal{Y}=\big\{1,2,3,\ldots,(n+1)^2-1\big\}$ be the set of all alphabets of the element in the feature vectors. Note again that the size of $\mathcal Y$ grows with $n$. Also, let $\mathbf{u}_i$ and $\mathbf{v}_i$ for $(1\leq i\leq n)$, be vectors of length $(n+1)^2-1$, given by
\begin{align}\label{eq_e48,eq_e49}
\mathbf{u}_i &= \bigg[0,\ldots,0,\dfrac{1}{i(i+1)},\dfrac{2}{i(i+1)},\ldots,\dfrac{i}{i(i+1)},\dfrac{i+1}{i(i+1)},\dfrac{i}{i(i+1)},\ldots,\dfrac{1}{i(i+1)},0,\ldots,0\bigg],\\
\mathbf{v}_i &= \bigg[0,\ldots,0,\dfrac{1}{i(i+1)},\dfrac{1}{i(i+1)},\ldots,\dfrac{1}{i(i+1)},\dfrac{1}{i(i+1)},0,\ldots,0\bigg].
\end{align}
The number of zeros in the left-hand sides of $\mathbf{u}_i$ and $\mathbf{v}_i$ is $i^2-1$. To generate a feature vector from the label $x_1(x_2)$, we generate the vector $y^n=(y_1,y_2,\ldots,y_n)$, where $y_k$ take values from the set $\mathcal{Y}$, with probability distribution $p_{_i}(y_i|x_1)=\mathbf{u}_i(y_i)$ $\big(p_{_i}(y_i|x_2)=\mathbf{v}_i(y_i)\big)$. 
\begin{figure}[!h]
	\centering
	\begin{tikzpicture}[scale=3]
	\draw (-2.5,0) -- (1.5,0) coordinate (x axis) node[below right] {Alphabets};
	\path (-0.5,-0.05) node(x) {$\ldots$};
	\foreach \x in {-2.5,-2.25,...,-0.75}
	\draw (\x,0) -- (\x,-0.05);
	\foreach \x in {-0.25,0}
	\draw (\x,0) -- (\x,-0.05);
	\draw (1.5,0) -- (1.5,-0.05);
	\draw (-2.5,0) -- (-2.25,0.5) -- (-2,0);
	\draw[<-] (-2.4,0.2) arc (30:90:0.22) node[left] {$p_{_1}(y_1|x_1)$};
	\draw (-1.75,0) -- (-1.25,0.3) -- (-0.75,0);
	\draw[<-] (-1.25,0.3) arc (30:90:0.22) node[left] {$p_{_2}(y_2|x_1)$};
	\path (0.625,-0.05) node(x) {$\ldots$};
	\draw (-0.25,0) -- (0.625,0.15) -- (1.5,0);
	\draw[<-] (0.625,0.15) arc (30:90:0.22) node[left] {$p_{_n}(y_n|x_1)$};
	
	\draw (-2.5,-1.25) -- (1.5,-1.25) coordinate (x axis) node[below right] {Alphabets};
	\path (-0.5,-1.3) node(x) {$\ldots$};
	\foreach \x in {-2.5,-2.25,...,-0.75}
	\draw (\x,-1.25) -- (\x,-1.3);
	\foreach \x in {-0.25,0}
	\draw (\x,-1.25) -- (\x,-1.3);
	\draw (1.5,-1.25) -- (1.5,-1.3);
	\draw (-2.5,-1.25) -- (-2.5,-0.75) -- (-2,-0.75) -- (-2,-1.25);
	\draw[<-] (-2.5,-1.05) arc (30:90:0.22) node[left] {$p_{_1}(y_1|x_2)$};
	\draw (-1.75,-1.25) -- (-1.75,-0.95) -- (-0.75,-0.95) -- (-0.75,-1.25);
	\draw[<-] (-1.15,-0.95) arc (30:90:0.22) node[left] {$p_{_2}(y_2|x_2)$};
	\path (0.625,-1.3) node(x) {$\ldots$};
	\draw (-0.25,-1.25) -- (-0.25,-1.1) -- (1.5,-1.1) -- (1.5,-1.25);
	\draw[<-] (0.625,-1.1) arc (30:90:0.22) node[left] {$p_{_n}(y_n|x_2)$};
	\end{tikzpicture}
	\caption[hey]
	{Illustration of the probability distributions $p_{_i}(y_i|x_1)$ (upper figure) and $p_{_i}(y_i|x_2)$ (lower figure), for $i=1,2,\ldots,n$.}
	\label{fig:3.2}
\end{figure}
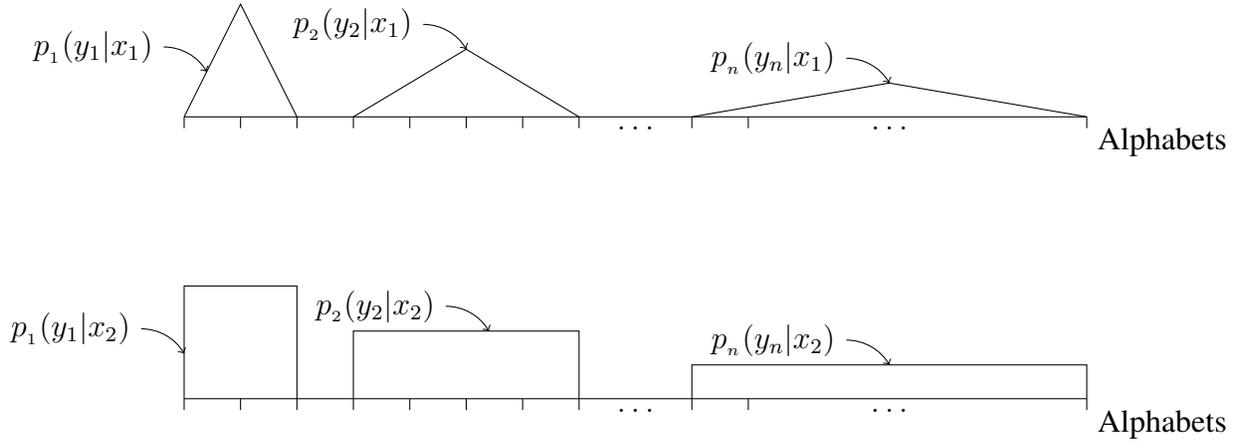

\begin{figure}[!h]
	\centering			\includegraphics[width=1\linewidth]{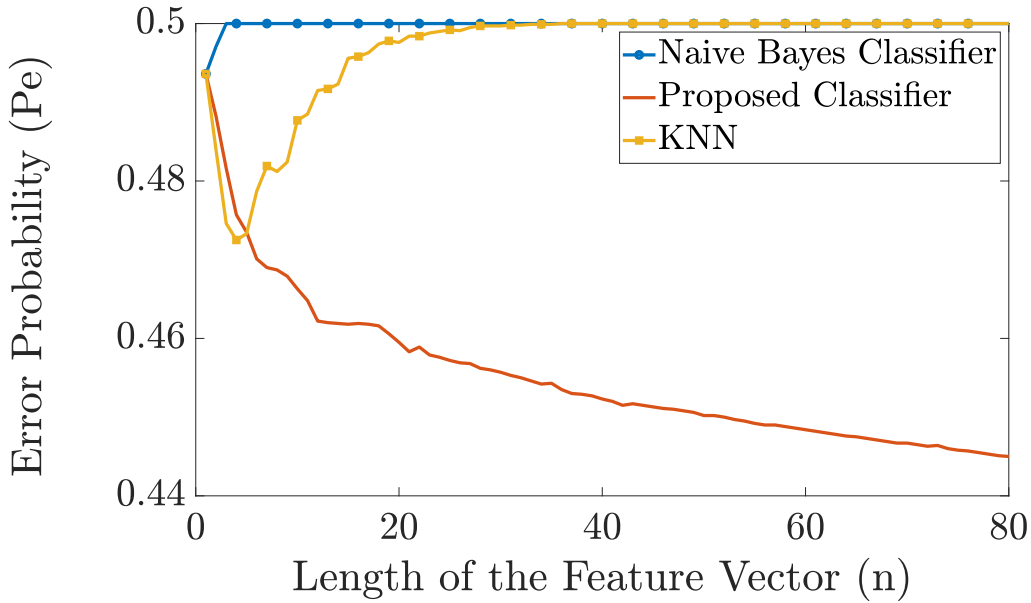}
	\caption{Comparison in error probability between the naive Bayes classifier and the proposed classifier $(T=1)$.} \label{fig:fig6}
\end{figure}

The simulation is carried out as follows. For each $n$, we generate one training feature vector for each label. Then we generate $250$ feature vectors for each label and pass it through our proposed classifier, the naive Bayes classifier and KNN and calculate the error probabilities. We change the length of the vector from $1$ to $80$ and repeat the simulation $250$ times and then plot the error probability.
Once again, there is a huge difference between the performance of the proposed classifier and the two benchmark classifiers. The error probability of the naive Bayes classifier is almost $0.5$ for all shown values of $n$ as it is susceptible to the issue with unseen alphabets in the training feature vector. The error probability of the KNN classifier is almost $0.5$ for all shown values of $n>30$ as it becomes susceptible to the ``curse of dimensionality''.
However, the error probability of our proposed classifier still decays continuously as $n$ increases.

Note that, in our numerical examples, we have compared our algorithm with the benchmark schemes on two extreme cases of i.non-i.d. vectors, referred to as ``overlapping'' and ``non-overlapping''. Any other i.non-i.d. vector can be represented as a combination of the ``overlapping'' and ``non-overlapping'' vectors. Since our algorithm works better than the benchmark schemes for small $t$ on both these cases, it will work better than the benchmark schemes on any combination between ``overlapping'' and ``non-overlapping'' vectors, i.e., for any other i.non-i.d. vectors.  

\section{Conclusion}\label{sec5}
In this paper, we investigated the important problem of classifying  feature vectors  with independent but non-identically distributed elements. For this problem, we proposed a classifier and derived an upper bound on  its error probability. Thereby, we showed that the proposed classifier is    asymptotically optimal   since its error probability goes to zero as the length of the input feature vectors grows. We showed that this asymptotic optimality is achievable even when one training feature vector per label is available. In the numerical examples, we compared the  proposed classifier with the naive Bayes classifier  and the KNN algorithm. Our numerical results show that the proposed classifier outperforms the benchmark classifiers  when the number of training data is small and the length of the input feature vectors is sufficiency large.
 
\appendix 
\subsection{Proof of Corollary 2}\label{A1}
The proof is almost identical to the proof of Theorem~\ref{th2}, however, here we derive a looser upper-bound on the error-probability than that in \eqref{eq_e9}, which is independent of $P_{\hat{y}^{nT}_i}$.

Without loss of generality we assume that $x_1$ is the input to $p_{_{Y^n|X}}(y^n|x)$ and $y^n$ is observed at the classifier.
	
	Let $\mathcal{B}^{\epsilon}_{k,l}$, for $1\leq k\leq|\mathcal{Y}|$ and $1\leq l\leq|\mathcal{X}|$, be a set defined as
	\begin{align}\label{eq_e51}
	\mathcal{B}^{\epsilon}_{k,l}=\Bigg\{\hat{y}^{nt}:\bigg|\dfrac{\mathcal{I}\big[\hat{y}^{nt}=y_k\big]}{nt}- \bar{\mathrm{p}}(y_k|x_l)\bigg|\leq\frac{\epsilon}{\sqrt[3]{t}}\Bigg\}.
	\end{align}
	
	Let $\mathcal{B}^{\epsilon}_l=\bigcap\limits_{k=1}^{|\mathcal{Y}|} \mathcal{B}^{\epsilon}_{k,l}$. For $\hat{y}_1^{nt}\in\mathcal{B}^{\epsilon}_1$, we have
	\begin{align}\label{eq_e53}
	\Bigg(\sum_{k=1}^{|\mathcal{Y}|}\bigg|\dfrac{\mathcal{I}[\hat{y}^{nt}_1=y_k]}{nt}-\bar{\mathrm{p}}(y_k|x_1)\bigg|^{r}\Bigg)^{1/r}\myleqa \Bigg(\sum_{k=1}^{|\mathcal{Y}|}\bigg(\frac{\epsilon}{\sqrt[3]{t}}\bigg)^r\Bigg)^{1/r},
	\end{align} 
	Using the same derivation as \eqref{eq_e21}, for any $y^n\in\mathcal{A}^{\epsilon}$ and for $\hat{y}^{nt}_1\in\mathcal{B}^{\epsilon}_1$, we have:
	\begin{align}\label{eq_e55}
		\Bigg(\sum_{k=1}^{|\mathcal{Y}|}\bigg|\dfrac{\mathcal{I}[y^n=y_k]}{n}-\dfrac{\mathcal{I}[\hat{y}_1^{nt}=y_k]}{nT}\bigg|^{r}\Bigg)^{1/r}\leq |\mathcal{Y}|^{1/r}\epsilon+|\mathcal{Y}|^{1/r}\frac{\epsilon}{\sqrt[3]{t}}.
	\end{align}
	On the other hand, same as the derivation in \eqref{eq_e24}, for each $i\neq 1$, we have:
	\begin{align}\label{eq_e58}
		\Bigg(\sum_{k=1}^{|\mathcal{Y}|}\bigg|\dfrac{\mathcal{I}[y^n=y_k]}{n}-\dfrac{\mathcal{I}[\hat{y}_i^{nt}=y_k]}{nt}\bigg|^{r}\Bigg)^{1/r}\geq &\big\lVert P_{\hat{y}^{nt}_i}-\bar{\mathrm{P}}_1\big\rVert_r-|\mathcal{Y}|^{1/r}\epsilon.
	\end{align}
	Now, for any $\hat{y}^{nt}_i\in\mathcal{B}^\epsilon_i$, we have
	\begin{align}\label{eq_e59}
	    &\big\lVert P_{\hat{y}^{nt}_i}-\bar{\mathrm{P}}_1\big\rVert_r+	\Bigg(\sum_{k=1}^{|\mathcal{Y}|}\bigg(\frac{\epsilon}{\sqrt[3]{t}}\bigg)^r\Bigg)^{1/r}\nonumber\\
	    &\mygeqa\Bigg(\sum_{k=1}^{|\mathcal{Y}|}\bigg|\dfrac{\mathcal{I}[\hat{y}^{nt}_i=y_k]}{nt}-\bar{\mathrm{p}}(y_k|x_1)\bigg|^{r}\Bigg)^{1/r}+\Bigg(\sum_{k=1}^{|\mathcal{Y}|}\bigg|\dfrac{\mathcal{I}[\hat{y}^{nt}_i=y_k]}{nt}-\bar{\mathrm{p}}(y_k|x_i)\bigg|^{r}\Bigg)^{1/r}\nonumber\\
	    &\mygeqb \Bigg(\sum_{k=1}^{|\mathcal{Y}|}\big|\bar{\mathrm{p}}(y_k|x_1)-\bar{\mathrm{p}}(y_k|x_i)\big|^{r}\Bigg)^{1/r},
	\end{align}
	where $(a)$ follows from \eqref{eq_e51} and $(b)$ is again due to the Minkowski inequality. The expression in \eqref{eq_e59}, can be written equivalently as
	\begin{align}\label{eq_e60}
	    \big\lVert P_{\hat{y}^{nt}_i}-\bar{\mathrm{P}}_1\big\rVert_r\geq\big\lVert \bar{\mathrm{P}}_i-\bar{\mathrm{P}}_1\big\rVert_r-|\mathcal{Y}|^{1/r}\frac{\epsilon}{\sqrt[3]{t}}.
	\end{align}
	where $i\neq1$. Using the bounds in \eqref{eq_e60} and \eqref{eq_e58}, for any $i\neq1$ we have
	\begin{align}\label{eq_e61}
	    \Bigg(\sum_{k=1}^{|\mathcal{Y}|}\bigg|\dfrac{\mathcal{I}[y^n=y_k]}{n}-\dfrac{\mathcal{I}[\hat{y}_i^{nt}=y_k]}{nt}\bigg|^{r}\Bigg)^{1/r}\geq\big\lVert \bar{\mathrm{P}}_i-\bar{\mathrm{P}}_1\big\rVert_r-|\mathcal{Y}|^{1/r}\epsilon\bigg(1+\frac{1}{\sqrt[3]{t}}\bigg).
	\end{align}

	Using the bounds in \eqref{eq_e55} and \eqref{eq_e61}, we now relate the left-hand sides of \eqref{eq_e55} and \eqref{eq_e61} as follows. As long as the following inequality holds for each $i\neq1$,
	\begin{align}\label{eq_e62}
	|\mathcal{Y}|^{1/r}\epsilon\bigg(1+\frac{1}{\sqrt[3]{T}}\bigg)<\lVert \bar{\mathrm{P}}_i-\bar{\mathrm{P}}_1\big\rVert_r-|\mathcal{Y}|^{1/r}\epsilon\bigg(1+\frac{1}{\sqrt[3]{t}}\bigg),
	\end{align}
	which is equivalent to the following for $i\neq1$
	\begin{align}\label{eq_e63}
		\epsilon<\dfrac{\big\lVert \bar{\mathrm{P}}_i-\bar{\mathrm{P}}_1\big\rVert_r}{2(1+t^{-1/3})|\mathcal{Y}|^{1/r}},
	\end{align}
	we have the following for $i\neq1$
	\begin{align}\label{eq_e64}
		\Bigg(\sum_{k=1}^{|\mathcal{Y}|}\bigg|\dfrac{\mathcal{I}[y^n=y_k]}{n}-\dfrac{\mathcal{I}[\hat{y}_1^{nt}=y_k]}{nt}\bigg|^{r}\Bigg)^{1/r}&\myleqa 	|\mathcal{Y}|^{1/r}\epsilon\bigg(1+\frac{1}{\sqrt[3]{t}}\bigg)\nonumber\\
		&\myleqqb \lVert \bar{\mathrm{P}}_i-\bar{\mathrm{P}}_1\big\rVert_r-|\mathcal{Y}|^{1/r}\epsilon\bigg(1+\frac{1}{\sqrt[3]{t}}\bigg)\nonumber\\&\myleqc 	\Bigg(\sum_{k=1}^{|\mathcal{Y}|}\bigg|\dfrac{\mathcal{I}[y^n=y_k]}{n}-\dfrac{\mathcal{I}[\hat{y}_i^{nt}=y_k]}{nt}\bigg|^{r}\Bigg)^{1/r},
	\end{align}
	where $(a)$, $(b)$, and $(c)$ follow from \eqref{eq_e55}, \eqref{eq_e62}, and \eqref{eq_e61}, respectively. Thereby, from \eqref{eq_e64}, we have the following for $i\neq1$
	\begin{align}\label{eq_e65}
		\Bigg(\sum_{k=1}^{|\mathcal{Y}|}\bigg|\dfrac{\mathcal{I}[y^n=y_k]}{n}-\dfrac{\mathcal{I}[\hat{y}_1^{nt}=y_k]}{nt}\bigg|^{r}\Bigg)^{1/r}\leq\Bigg(\sum_{k=1}^{|\mathcal{Y}|}\bigg|\dfrac{\mathcal{I}[y^n=y_k]}{n}-\dfrac{\mathcal{I}[\hat{y}_i^{nt}=y_k]}{nt}\bigg|^{r}\Bigg)^{1/r},
	\end{align}
	or equivalently as  
	\begin{align}\label{eq_e66}
		\big\lVert P_{y^n}-P_{\hat{y}_1^{nt}}\big\rVert_r<\big\lVert P_{y^n}-P_{\hat{y}_i^{nt}}\big\rVert_r.
	\end{align}
	Once again, we obtained that if there is an $\epsilon$ for which \eqref{eq_e63} holds for $i\neq1$  and for that $\epsilon$ there are sets $\mathcal{A}^{\epsilon}$ and $\mathcal{B}^{\epsilon}_i$ for which $y^n\in\mathcal{A}^{\epsilon}$ and $\hat{y}^{nt}_j\in\mathcal{B}^{\epsilon}_l$ for all $1\leq l\leq|\mathcal{X}|$, then \eqref{eq_e66} holds for $i\neq1$, and thereby our classifier will detect that $x_1$ is the correct label. Using this, we can upper-bound the error probability as
	\begin{align}\label{eq_e67}
	\mathbb{P}_\e&=1-\pr\big\{\hat{x}_1= x_1\big\}\nonumber\\&\leq 1-\pr\Bigg\{\big(y^n\in\mathcal{A}^{\epsilon}\big)\cap\bigg(\bigcap_{j=1}^{|\mathcal{X}|}\hat{y}^{nt}_l\in\mathcal{B}^{\epsilon}_l\bigg)\bigg\lvert\epsilon\in\mathcal{S}\Bigg\},
	\end{align}
	where $\mathcal{S}$ is a set defined as
	\begin{align}\label{eq_e68}
		\mathcal{S}=\Bigg\{\epsilon:\epsilon\leq\min_{\substack{i\\i\neq 1}}\dfrac{\big\lVert \bar{\mathrm{P}}_i-\bar{\mathrm{P}}_1\big\rVert_r}{(2+t^{-1/3})|\mathcal{Y}|^{1/r}}\Bigg\}.
	\end{align}
	The right-hand side of \eqref{eq_e67} can be upper-bounded as
	\begin{align}\label{eq_e69}
	1-\pr\Bigg\{\big(y^n\in\mathcal{A}^{\epsilon}\big)\cap\bigg(\bigcap_{l=1}^{|\mathcal{X}|}\hat{y}^{nt}_l\in\mathcal{B}^{\epsilon}_j\bigg)\bigg\lvert\epsilon\in\mathcal{S}\Bigg\}&=\pr\Bigg\{\big(y^n\notin\mathcal{A}^{\epsilon}\big)\cup\bigg(\bigcup_{l=1}^{|\mathcal{X}|}\hat{y}^{nt}_l\notin\mathcal{B}^{\epsilon}_l\bigg)\bigg\lvert\epsilon\in\mathcal{S}\Bigg\}\nonumber\\
	&\myleqa\pr\big\{y^n\notin\mathcal{A}^{\epsilon}\lvert\epsilon\in\mathcal{S}\big\}\nonumber\\&\quad+\sum_{l=1}^{|\mathcal{X}|}\pr\Big\{\hat{y}^{nt}_l\notin\mathcal{B}^{\epsilon}_l\big\lvert\epsilon\in\mathcal{S}\Big\},
	\end{align}
	Using the same derivation as \eqref{eq_e35}, we have:
	\begin{align}\label{eq_e72}
		\pr\big\{y^n\notin\mathcal{A}^{\epsilon}\lvert\epsilon\in\mathcal{S}\big\}\leq 2|\mathcal{Y}|\e^{-2n\epsilon^2}.
	\end{align}
Similarly, we have the following result for the second expression in the right-hand side of \eqref{eq_e69}, same as the derivation in \eqref{eq_e39}
	\begin{align}\label{eq_e76}
	\pr\Big\{\hat{y}^{nt}_l\notin\mathcal{B}^{\epsilon}_l\big\lvert\epsilon\in\mathcal{S}\Big\}\leq 2|\mathcal{Y}|\e^{-2nt^{1/3}\epsilon^2}.
	\end{align}
	Inserting \eqref{eq_e72} and \eqref{eq_e76} into \eqref{eq_e69}, and then inserting \eqref{eq_e69} into \eqref{eq_e67}, we obtain the following upper-bound for the error probability
	\begin{align}\label{eq_e77}
	\mathbb{P}_\e\leq 2|\mathcal{Y}|\e^{-2n\epsilon^2}+2|\mathcal{X}||\mathcal{Y}|\e^{-2nt^{1/3}\epsilon^2},
	\end{align}
	where 
	\begin{align}\label{eq_e78}
		\epsilon = \min_{\substack{i,j\\i\neq j}}\dfrac{\big\lVert \bar{\mathrm{P}}_i-\bar{\mathrm{P}}_j\big\rVert_r}{2(1+t^{-1/3})|\mathcal{Y}|^{1/r}},
	\end{align}
	Now, if $|\mathcal{Y}|\leq n^m$, \eqref{eq_e77} can be written as
\begin{align}\label{eq_e79}
\mathbb{P}_\e&\leq 2|\mathcal{Y}|\e^{-2n\epsilon^2}+2|\mathcal{X}||\mathcal{Y}|\e^{-2nt^{1/3}\epsilon^2}\nonumber\\
&\leq 2n^m\exp\Bigg(-2n\min_{\substack{i,j\\i\neq j}}\dfrac{\big\lVert \bar{\mathrm{P}}_i-\bar{\mathrm{P}}_j\big\rVert^2_r}{2(1+t^{-1/3})^2n^{2m/r}}\Bigg)\nonumber\\
&+2|\mathcal{X}|n^m\exp\Bigg(-2nt^{1/3}\min_{\substack{i,j\\i\neq j}}\dfrac{\big\lVert \bar{\mathrm{P}}_i-\bar{\mathrm{P}}_j\big\rVert^2_r}{2(1+t^{-1/3})^2n^{2m/r}}\Bigg)\nonumber\\
&\leq \mathcal{O}\bigg(n^m\exp\Big(-n^{1-\frac{2m}{r}}\Big)\bigg).
\end{align}
According to \eqref{eq_e79}, for a fixed $r>2m$, the right-hand side of \eqref{eq_e79} goes to zero as $n\to\infty$, and thereby, the classifier is asymptotically optimal.		

\bibliographystyle{IEEEtran}
\bibliography{citations}


\end{document}